\documentclass{article}

\PassOptionsToPackage{numbers,compress}{natbib}
\usepackage[main,final]{neurips_2026}
\usepackage{enumitem}
\usepackage[utf8]{inputenc} 
\usepackage[T1]{fontenc}    
\usepackage{hyperref}       
\usepackage{url}            
\usepackage{booktabs}       
\usepackage{amsfonts}       
\usepackage{nicefrac}       
\usepackage{microtype}      
\usepackage{xcolor}         
\usepackage{todonotes}
\usepackage{amsmath} 
\usepackage{xspace}
\usepackage{multirow}
\usepackage{graphicx}
\usepackage{cleveref}
\usepackage{gensymb}   
\usepackage{amssymb}
\usepackage{wrapfig}
\usepackage{booktabs}
\usepackage{longtable} 

\usepackage{amsthm}   

\newtheorem{definition}{Definition}
\newtheorem{proposition}{Proposition}[section]
\newtheorem{lemma}{Lemma}

\newtoggle{final}
\toggletrue{final}
\iftoggle{final}{}{\setlength{\marginparwidth}{2cm}}
\newcommand{\hjs}[1]{\iftoggle{final}{#1}{{\color{purple} #1}}}

\newcommand{\ours}{DSSP\xspace}
\title{DSSP: Diffusion State Space Policy \\ with Full-History Encoding}

\author{%
  Zhiyuan Guan\textsuperscript{1} \quad
  Jianshu Hu\textsuperscript{1} \quad
  Han Fang\textsuperscript{1} \quad
  Yunpeng Jiang\textsuperscript{1} \\
  \textbf{Yize Huang\textsuperscript{2} \quad
  Shujia Li\textsuperscript{1} \quad
  Xiao Li\textsuperscript{2} \quad
  Yutong Ban\textsuperscript{1}}\thanks{Corresponding author: yban@sjtu.edu.cn.} \\
  \\
  \textsuperscript{1}Global College, Shanghai Jiao Tong University \\
  \textsuperscript{2}School of Mechanical Engineering, Shanghai Jiao Tong University
}

\begin{document}

\maketitle

\begin{abstract}
Diffusion-based imitation learning has shown strong promise for robot manipulation. However, most existing policies condition only on the current observation or a short window of recent observations, limiting their ability to resolve history-dependent ambiguities in long-horizon tasks.
To address this, we introduce \ours, a history-conditioned Diffusion State Space Policy that enables efficient, full-history conditioning for robot manipulation.
Leveraging the continuous sequence modeling properties of State Space Models (SSMs), our history encoder effectively compresses the entire observation stream into a compact context representation.
To ensure this context preserves critical information regarding future state evolution, the encoder is optimized with a dynamics-aware auxiliary training objective.
This high-level context representation is then seamlessly fused with recent state observations to form a hierarchical conditioning mechanism for action generation. Furthermore, to maintain architectural consistency and minimize GPU memory overhead, we also instantiate the diffusion backbone itself using an SSM.
Extensive experiments across simulation benchmarks and real-world manipulation tasks show that \ours achieves state-of-the-art performance with a significantly smaller model size,
demonstrating superior efficiency of the hierarchical conditioning in capturing crucial information as the history length increases. Project page: \url{https://yuaniw.github.io/DSSP-Page/}.

\end{abstract}

\section{Introduction}

Deploying robots in complex, unstructured environments requires policies capable of reasoning over multi-step tasks and high-dimensional sensory inputs.
Imitation learning~\citep{levine_end--end_2016,zeng_transporter_2022,zhao_learning_2023} has emerged as a central paradigm for acquiring such skills from expert demonstrations.
Among recent imitation learning approaches, diffusion policies~\citep{chi_diffusion_2024} have shown strong potential for modeling complex and multi-modal action distributions.
Recent works further improve diffusion-based policies through expressive denoising backbones~\citep{yan_maniflow_2025}, multi-modal conditioning representations~\citep{ze_3d_2024,lu_hmathbf3dp_2025}, and improved action generation formulations~\citep{zhang2025flowpolicy,sheng_mp1_2025}.
However, most policies still rely on short-context conditioning, predicting future actions from the current observation or a short observation window.


This short-context conditioning becomes a critical bottleneck in long-horizon manipulation. In multi-stage tasks, visually similar observations may correspond to different task progress, such as whether an object has already been moved, whether a container has been filled, or which object was previously placed in a buffer location. 
The 
desired actions therefore depend not only on local visual details, but also on sparse historical events that reveal task progress and previous interactions.
Without access to such history, policies may fail to disambiguate temporally aliased states, leading them to ignore completed subgoals or undo previous actions.
\begin{figure*}[th]
    \centering
    \includegraphics[width=\textwidth]{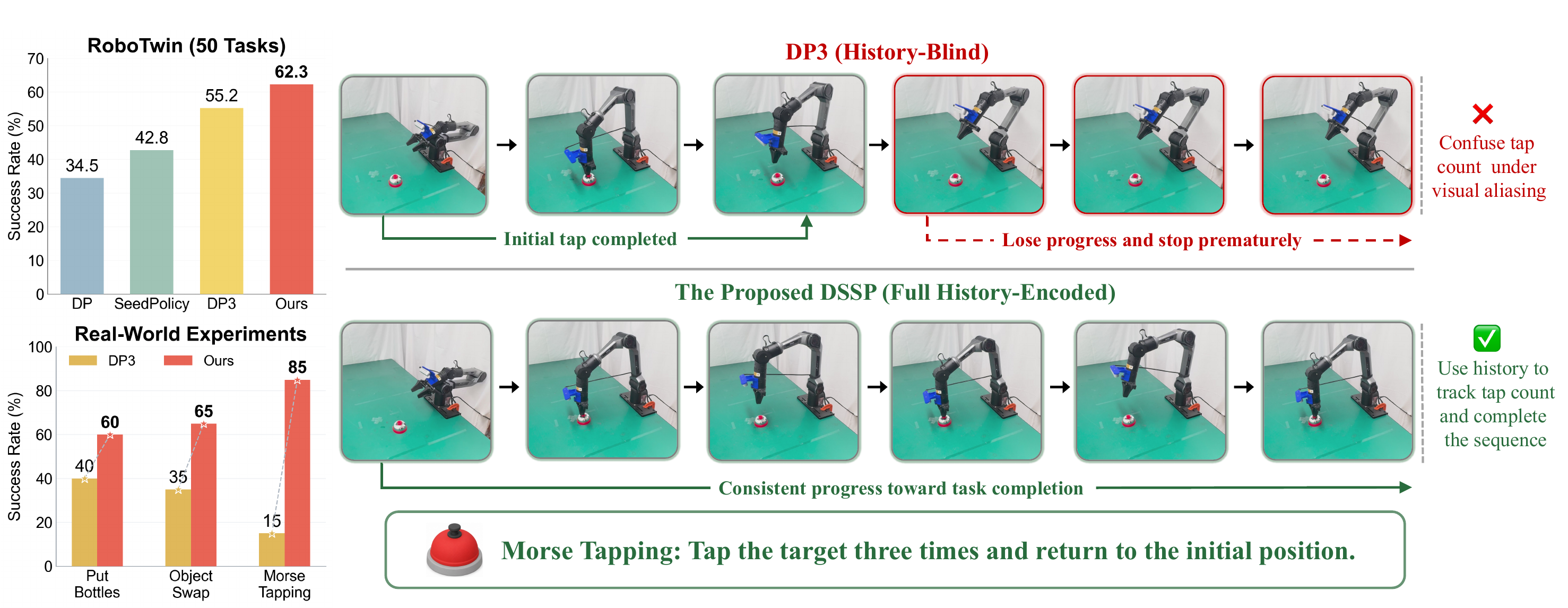}
    \caption{
    The proposed \textbf{\ours{}} leverages full-history context to resolve visual aliasing when history-blind baselines lose track of task progress, enabling consistent execution in long-horizon tasks.
    \ours{} achieves superior success rates across both simulation tasks and real-world experiments.
    }
    \label{fig:teaser}
\end{figure*}

Existing methods have attempted to address the limitations of short-context policies by adding temporal context.
A direct solution extends the observation horizon or feeds longer observation histories into the policy, but this increases memory and inference cost and may introduce redundant visual inputs
or spurious correlations~\citep{wen_fighting_2020,torne_learning_2025,zheng_tracevla_2025,mark_bpp_2026}.
Attention-based and multi-frame VLA methods provide more flexible access to past observations, yet remain costly because they explicitly process additional context as the horizon grows~\citep{li_cronusvla_2025,jang_contextvla_2025,koo_hamlet_2026}. To reduce this overhead, recent works therefore explore more compact memory forms, including keyframes~\citep{mark_bpp_2026,wei2025cyclemanip}, visual traces~\citep{zheng_tracevla_2025}, point tracks~\citep{chen_history-aware_2026}, or recurrent latent states~\citep{zhou_mtil_2025}.
However, compactness does not ensure task relevance: compressed memories may discard progress-critical events or focus on nuisance visual changes.
These limitations motivate a streaming history mechanism that efficiently compresses long-horizon context while preserving events predictive of future state.


To fulfill these requirements, State Space Models (SSMs) offer a well-suited architecture for history encoding.
Their recurrent formulation maintains a compact hidden state that is updated online as new observations arrive, enabling linear-time aggregation of long histories without re-encoding the entire past.
Moreover, modern selective SSMs such as Mamba \citep{gu_mamba_2024} make the state update input-dependent, allowing the model to adaptively decide what to preserve or discard. This content-dependent memory update is particularly suitable for long-horizon manipulation, where visually similar observations may correspond to different task stages and require different actions. 
Building on this insight, we propose \ours{}, a full-history encoding diffusion state-space policy for long-horizon manipulation.
Following the use of SSMs in diffusion-based robot policies~\citep{cao_mamba_2025,jia_mail_2024,yoo_robossm_2025}, \ours{} instantiates the action denoiser with an SSM backbone, but conditions it on a compact online memory compressed from the full multi-modal observation stream together with immediate state information.
A dynamics-aware auxiliary objective encourages this memory to preserve cues predictive of future states.
We further decouple task conditioning from diffusion-step modulation by injecting observation-derived conditions as a prefix and timestep information through Adaptive Layer Normalization (AdaLN)~\citep{peebles_scalable_2023}.


Our contributions are summarized as follows:
\begin{itemize}
[leftmargin=0.8em,itemsep=0.2em,topsep=0.2em]
\item \textbf{Hierarchically conditioned state space policy.}
    We design a policy framework instantiated with a state space model and a hierarchical conditioning mechanism.
    Specifically, learned context representations and immediate state representations are fused via prefix conditioning, while the diffusion timestep is decoupled and injected independently through AdaLN.

    \item \textbf{Full-history context learning.}
    We propose a causal state-space history encoder that maintains a compact context representation by recurrently integrating incoming multi-modal observations.
    Together with a dynamics-aware auxiliary objective, the encoder summarizes the full observation history while retaining task-relevant events predictive of future states.


    \item \textbf{Comprehensive evaluation.}
    We evaluate \ours{} on extensive simulated and real-world long-horizon manipulation tasks, showing improved success rates, effective history summarization, perturbation robustness, and efficient inference with increasing history length.
\end{itemize}
\section{Related Work}\label{sec:related}
\textbf{Imitation Learning for Robot Manipulation.}
Imitation learning enables policies to acquire skills from expert demonstrations. 
While early behavior cloning methods directly predict actions from observations, recent approaches improve closed-loop consistency by modeling temporally extended or structured actions, such as action chunks~\citep{zhao_learning_2023} and multimodal action representations~\citep{shafiullah_behavior_2022,lee_behavior_2024}. 
Diffusion-based policies further generate continuous action trajectories via conditional denoising~\citep{chi_diffusion_2024}, with 3D Diffusion Policy extending this framework to point-cloud observations~\citep{ze_3d_2024}. 
Recent variants improve diffusion policies in perception~\citep{lu_hmathbf3dp_2025}, efficiency~\citep{fang2026omponestepmeanflowpolicy,sheng_mp1_2025}, and generation quality~\citep{yan_maniflow_2025,ma_cdp_2025}. 
Our method follows this paradigm, using conditional denoising for action generation.

\textbf{Long-Horizon and History-Aware Policy Learning.} Most existing robot policies condition action prediction on the current observation or a short window of recent observations, which is often sufficient for short-horizon or near-Markovian tasks.
In long-horizon manipulation, resolving temporal ambiguity requires observation history. However, naively stacking past observations often degrades performance due to redundancy and causal confusion~\citep{haan_causal_2019, wen_fighting_2020, wen_keyframe-focused_2021, swamy_sequence_2023}. Recent methods therefore explore more selective uses of history, such as regularizing long-context policies with past-token prediction~\citep{torne_learning_2025}, selecting task-relevant keyframes from history~\citep{mark_bpp_2026}, using demonstration trajectories as in-context prompts~\citep{xie_-context_nodate}, compressing long observation histories into an evolving latent state~\citep{gui_seedpolicy_2026, zhou_mtil_2025}, leveraging history sampling and progress prediction for repetitive manipulation~\citep{wei2025cyclemanip}, or incorporating previous action trajectories for stateful diffusion-based control~\citep{liu2024diff}.
In contrast, our method learns a compact, dynamics-aware history context that preserves future-relevant information for effective long-horizon policy learning.

\textbf{State Space Models for Robot Policies.}
State space models have recently shown strong potential for sequence modeling by maintaining an evolving latent state over time. Mamba~\citep{gu_mamba_2024, dao_transformers_2024, lahoti_mamba-3_2026} further improves this paradigm with selective state updates and hardware-aware parallel scan, enabling efficient long-sequence modeling. Inspired by these properties, recent robotic learning methods have adopted Mamba as a backbone for imitation learning and diffusion-based action generation~\citep{oh2026dispodiffusionssmbasedpolicy, jia_mail_2024, cao_mamba_2025, jia_x-il_2025}, or as a recurrent encoder for long observation histories in temporally ambiguous tasks~\citep{tsuji_mamba_2025, zhou_mtil_2025}. These findings demonstrate the effectiveness of Mamba for modeling robot trajectories and temporal dependencies. In this paper, we introduce a unified Mamba-based policy that integrates dynamics-aware history encoding with diffusion-based action generation.

\section{Preliminaries}
\label{subsec:problem_setup}

\textbf{Problem Formulation.}
We formulate long-horizon 3D manipulation as a Partially Observable Markov Decision Process (POMDP), defined by the tuple $\mathcal{M} = (\mathcal{S}, \mathcal{A}, \mathcal{T}, \Omega, \mathcal{O})$. 
At any time step $t$, the agent receives an observation $o_t \in \Omega$ generated from the unobserved true state $s_t \in \mathcal{S}$ via the observation function $\mathcal{O}(o_t \mid s_t) = P(o_t \mid s_t)$. 
The environment evolves according to transition dynamics $\mathcal{T}(s_{t+1} \mid s_t, a_t) = P(s_{t+1} \mid s_t, a_t)$ given action $a_t \in \mathcal{A}$. 
Since $o_t$ is insufficient to infer $s_t$, we define the interaction history as $h_t = (o_0, a_0, \dots, a_{t-1}, o_t) \in \mathcal{H}$, where $\mathcal{H}$ denotes the full history space. 
Our goal is to learn a history-dependent policy $\pi_\theta(a_t \mid h_t)$ imitating an expert policy $\pi_E$ with expert trajectories $\mathcal{D}_E = \{ \zeta_1, \dots, \zeta_N \}$, where each trajectory is a sequence $\zeta = (o_0, a_0, \dots, o_T)$.
Due to the page limit, we include a detailed introduction to diffusion policy and SSM in Appendix~\ref{app:preliminaries}.

\textbf{Observation Aliasing.}
A non-injective observation function alone does not imply that history is useful for
imitating the expert: the aliased states may require the same action, or the expert may
use information that is absent from the recorded history.  We therefore use the
following decision-relevant notion.  In the definition, capital letters denote random
variables under the expert trajectory distribution, and $O_t$ is a component (and
hence a measurable function) of $H_t$.
\begin{definition}[History-resolvable observation aliasing]
\label{def:history_resolvable_aliasing}
History-resolvable observation aliasing occurs when the conditional expert-action law
given the available history is not almost surely determined by the current observation:
\begin{equation}
    P_E(A_t \in \cdot \mid H_t)
    \neq P_E(A_t \in \cdot \mid O_t)
    \quad\text{on a set of histories of positive probability}.
    \label{eq:distributional_aliasing}
\end{equation}
Equivalently,
\begin{equation}
    I_E(A_t;H_t\mid O_t)
    := \mathbb{E}\!\left[
        \operatorname{KL}\!\left(
        P_E(A_t\in\cdot\mid H_t)
        \,\middle\|\,
        P_E(A_t\in\cdot\mid O_t)
        \right)
    \right] > 0.
    \label{eq:conditional_mi}
\end{equation}
\end{definition}
This definition concerns the information actually available to the learned policy.
If the expert uses an unrecorded variable that cannot be predicted from $H_t$ beyond
$O_t$, then \eqref{eq:distributional_aliasing} does not hold and no strict benefit from
history is claimed.  Conversely, when it does hold, a history-conditioned diffusion
policy can represent the distinct conditional action distributions rather than their
observation-conditioned mixture; appendix \ref{app:theoretical_analysis} makes this statement
precise at the level of optimal population diffusion loss.

\section{Method}
\label{sec:method}
\begin{figure}[!t] 
    \centering 
    \includegraphics[width=\linewidth]{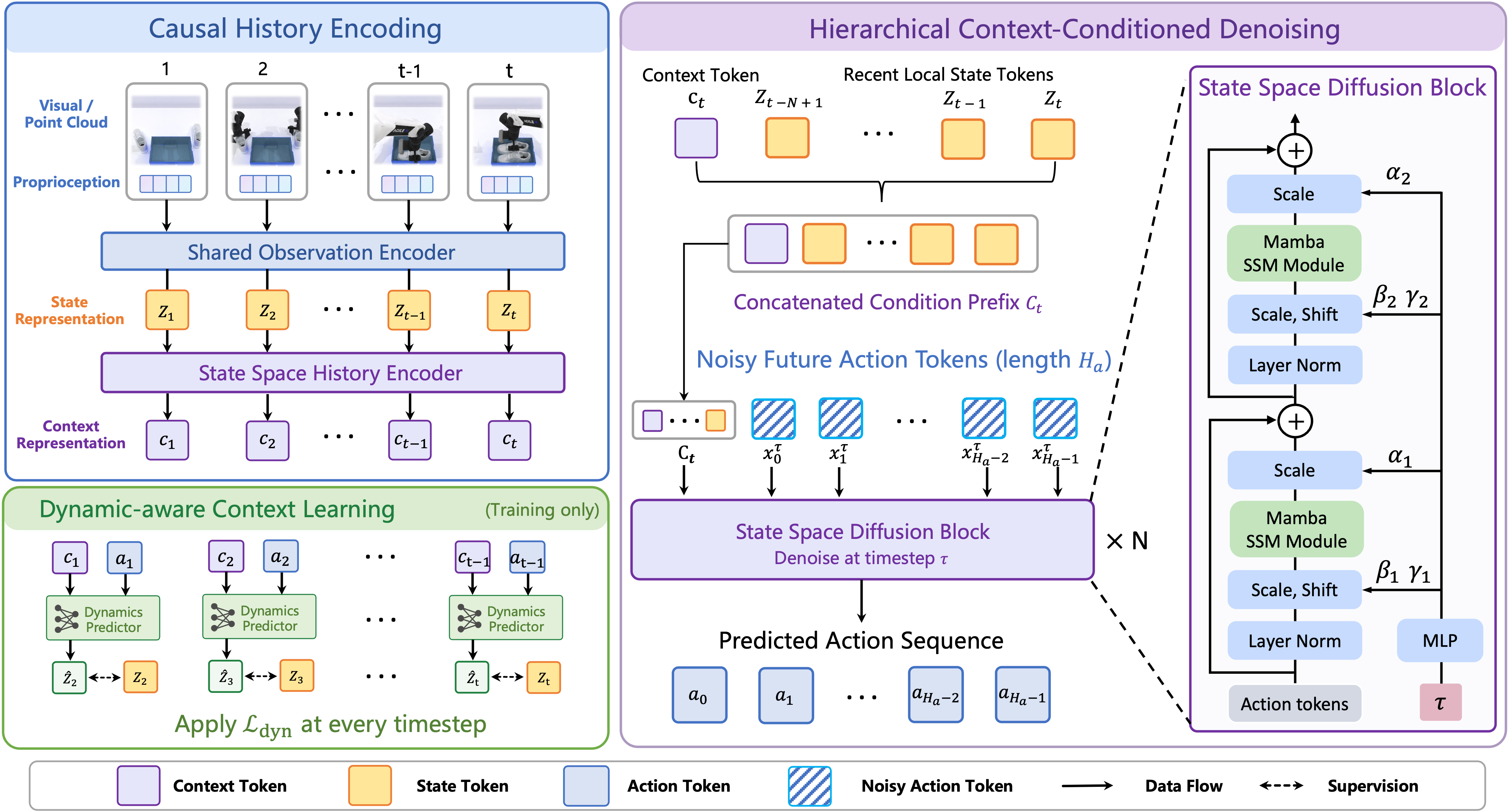} 
    \caption{
        \textbf{Overview of \ours{}.}
        \ours{} summarizes past multi-modal observations into a compact context token using a state-space history encoder. A dynamics-aware auxiliary loss encourages this token to retain historical information predictive of future state evolution. The learned context token is then combined with recent state tokens as a hierarchical prefix condition for a state-space diffusion denoiser to generate future actions.
        }
    \label{fig:method_overview}
\end{figure}




To effectively leverage history information within a robot manipulation policy, we propose diffusion state space policy (\ours), a full-history conditioned diffusion policy (illustrated in \Cref{fig:method_overview}).
Our approach instantiates the diffusion model with an SSM backbone and a dual-level conditioning mechanism that integrates high-level context with low-level state representations.
The context representation serves as a compact encoding of historical multi-modal observations.
To ensure this representation captures temporal dependencies, we introduce an auxiliary dynamics-aware loss focused on future state prediction. \Cref{sec:history_encoding} introduces long-horizon context learning, including causal history encoding and dynamics-aware context learning; and \Cref{sec:context_conditioned_diffusion} formalizes the hierarchical conditioning mechanism and diffusion-based action generation policy.

\subsection{Long-Horizon Context Learning}
\label{sec:history_encoding}

Long-horizon manipulation requires historical context to resolve task-level ambiguities; otherwise, policies may experience perceptual aliasing, such as getting trapped in loops during repetitive wiping tasks. 
Because existing conditioning approaches like observation stacking or keyframe extraction often suffer from computational redundancy or overlook subtle causal events, we introduce a learnable history encoder that compresses the full history into a compact, task-relevant context representation.

To this end, our design is guided by two key principles. First, we employ a causal SSM to efficiently process the streaming observation history and extract a temporally integrated memory representation. Second, to ensure that this representation preserves historical cues useful for future decisions, we shape the latent space with a dynamics-aware auxiliary objective.

\textbf{Causal History Encoding.}
We first introduce how to build the context representation for history encoding.
Let $o_t$ denote the multi-modal observation at timestep $t$, 
comprising visual inputs and robot proprioceptive states.
We project this raw observation into a state representation $z_t$:
\begin{equation}
\label{eq:observation_encoder}
z_t = E_{\mathrm{obs}}(o_t),
\end{equation}
where $E_{\mathrm{obs}}$ contains parallel visual and proprioceptive encoders.
To capture long-horizon temporal dependencies, a causal history encoder $G$ processes the sequence of step-wise state representations $z_{1:t}$ into the temporally-integrated context representation $\tilde{z}_{1:t}$:
\begin{equation}
\label{eq:history_encoding}
\tilde{z}_{1:t} = G_\psi(z_{1:t}).
\end{equation}
A critical design choice is the architecture of the history encoder $G$, which must produce a compact yet effective history representation.
This design must satisfy two primary criteria: (1) maintaining a scalable computation when processing extended temporal horizons, and (2) distilling a representation that selectively retains salient events rather than passively memorizing the entire observation stream.

To meet these requirements, we instantiate the history backbone using a State-Space Model (SSM) and define the context representation $c_t$ as the final output token of the encoded sequence:
\begin{equation}
   c_t = \tilde{z}_t.
\end{equation}
SSMs support streaming histories with linear-time complexity, and we use Mamba as the history encoder for input-dependent selective updates.
This allows the encoder to filter redundant observations while preserving sparse task-relevant events, such as object-state transitions, contact changes, or subgoal completion.
The resulting context $c_t$ serves as a compressed memory for action generation.

\textbf{Dynamics-Aware Context Learning.}
\label{sec:dynamics_context_learning}
Compressing the observation history into $c_t$ does not by itself guarantee that the representation preserves historical information relevant to future decisions.
For long-horizon manipulation, the context representation should encode not only past observations, but also history-dependent cues that are predictive of future state evolution.
To encourage this property, we introduce a dynamics-aware auxiliary objective applied to the context representation at each timestep.
Given the context representation $c_t$ and the executed action $a_t$ at time $t$, we train a lightweight dynamics predictor $g_{\phi}$ to predict the next state representation:
\begin{equation}
\label{eq:dynamics_prediction}
z_{t+1} = E_{\mathrm{obs}}(o_{t+1}),
\qquad
\hat{z}_{t+1} = g_{\phi}(c_t, a_t).
\end{equation}
We supervise this prediction with a cosine similarity loss:
\begin{equation}
\label{eq:dynamics_loss}
\mathcal{L}_{\mathrm{dyn}}(\psi, \phi) = \mathbb{E}_{\zeta \sim \mathcal{D}_E, t \sim [0, T-1]} \left[ 1 - \cos\left( g_\phi(c_t, a_t), \mathrm{sg}(z_{t+1}) \right) \right],
\end{equation}
where $\mathrm{sg}(\cdot)$ denotes the stop-gradient operation, and the expectation is taken over expert trajectories $\zeta \sim \mathcal{D}_E$ and trajectory timesteps $t$.
This objective encourages the context representation $c_t$ to retain action-relevant historical information by requiring it to support prediction of the next state under the executed action.
The learned context representation therefore provides a more informative conditioning signal for downstream action generation. Additional runtime diagnostics further demonstrate that the learned dynamics predictor captures
action-dependent state transitions in Appendix~\ref{app:dynamics_analysis}.

\subsection{Hierarchical Context-Conditioned Denoising}
\label{sec:context_conditioned_diffusion}

Given the learned history representation, the remaining question is how to inject it into a policy.
Long-horizon manipulation requires both long-term task-progress information and recent local observations: the former resolves visual ambiguity across task stages, while the latter provides fine-grained control cues.
Therefore, we propose a diffusion policy utilizing a hierarchical conditioning mechanism that integrates the context representation with immediate state observations.
We organize these signals as a causal prefix to the noisy action sequence, allowing long-term context to contextualize recent observations before being propagated to the action tokens during denoising.
For efficiency, we instantiate the diffusion backbone with a compact SSM, yielding a lightweight policy while maintaining a unified SSM-based architecture for both history encoding and action denoising.


\textbf{Hierarchical Prefix Conditioning.}
To preserve long-term progress information and local control cues, we condition the denoising model on a hierarchical prefix.
Let $x_0=\mathbf{a}_{t:t+H_a-1}$ denote the clean future action trajectory, and let $x_\tau$ denote its noisy version at diffusion step $\tau$.
We construct the condition sequence as
\begin{equation}
\label{eqs:prefix_condition}
C_t = [c_t, z_{t-N+1}, \dots, z_t],
\end{equation}
where $c_t$ is the context representation produced by the history encoder, $z_{t-N+1:t}$ are the most recent state tokens, and $N$ is the local observation window size.
We prepend $C_t$ to the noisy action sequence as a prefix condition.
The resulting sequence is processed by the SSM denoising backbone:
\begin{equation}
\label{eq:denoising_prediction}
\hat{x}_0 = f_\theta(x_\tau,\tau,C_t).
\end{equation}
In this design, $c_t$ captures long-horizon progress and $z_t$ retains local geometric and proprioceptive details. Thus, the policy leverages historical context without losing manipulation precision.

\textbf{Causal Action Denoising.}
The hierarchical prefix formulation casts action denoising as a causal prefix-conditioned sequence modeling.
Ordering the sequence as $[c_t, z_{t-N+1}, \dots, z_t, x_\tau]$ allows long-term context to contextualize recent states before propagating information to noisy action tokens.
We instantiate $f_\theta$ with a Mamba backbone, whose recurrent selective state updates naturally match this left-to-right conditioning flow while providing linear scaling for iterative diffusion sampling.

\textbf{Timestep-Decoupled Action Denoising.}
To provide a stable conditioning signal throughout iterative denoising, we decouple timestep modulation from prefix conditioning.
Since the diffusion timestep \hjs{$\tau$} describes the noise level of the action trajectory,
we inject the timestep embedding through AdaLN only into the action tokens while keeping the prefix condition unchanged.
This keeps $C_t$ as a stable representation of history and local observations throughout the denoising process.
Meanwhile, the actions remain aware of the current noise level and can progressively refine the predicted action sequence.
This design separates task conditioning from diffusion-step modulation.

\textbf{Training.}
The policy is optimized with two objectives: the diffusion reconstruction loss for action generation and the dynamics-aware auxiliary loss for context learning.
For an action window starting at time $t$, we construct the hierarchical condition $C_t$ causally from observations up to time $t$ and train the denoising model to predict the clean action trajectory:

\begin{equation}
\label{eq:diffusion_loss}
\mathcal{L}_{\mathrm{diff}}(\theta, \psi) = \mathbb{E}_{\zeta \sim \mathcal{D}_E, t, \tau, \epsilon} \left[ \left\| f_\theta(x_\tau, \tau, C_t) - x_0 \right\|_2^2 \right],
\end{equation}
where the expectation is taken over expert trajectories $\zeta \sim \mathcal{D}_E$, trajectory timesteps $t \sim \mathcal{U}(0, T-H_a)$, diffusion steps $\tau \sim \mathcal{U}(1, L)$, and Gaussian noise $\epsilon \sim \mathcal{N}(0, \mathbf{I})$. 
Accordingly, the overall objective is:
\begin{equation}
\label{eq:total_loss}
\mathcal{L(\theta, \psi, \phi)}
=
\mathcal{L}_{\mathrm{diff}}(\theta, \psi)
+
\lambda \mathcal{L}_{\mathrm{dyn}}(\psi, \phi)
\end{equation}
where $\lambda$ balances action denoising and context representation learning.

\begin{figure*}[t]
    \centering
     \includegraphics[width=\textwidth, trim=0cm 0.0cm 0cm 0cm, clip]{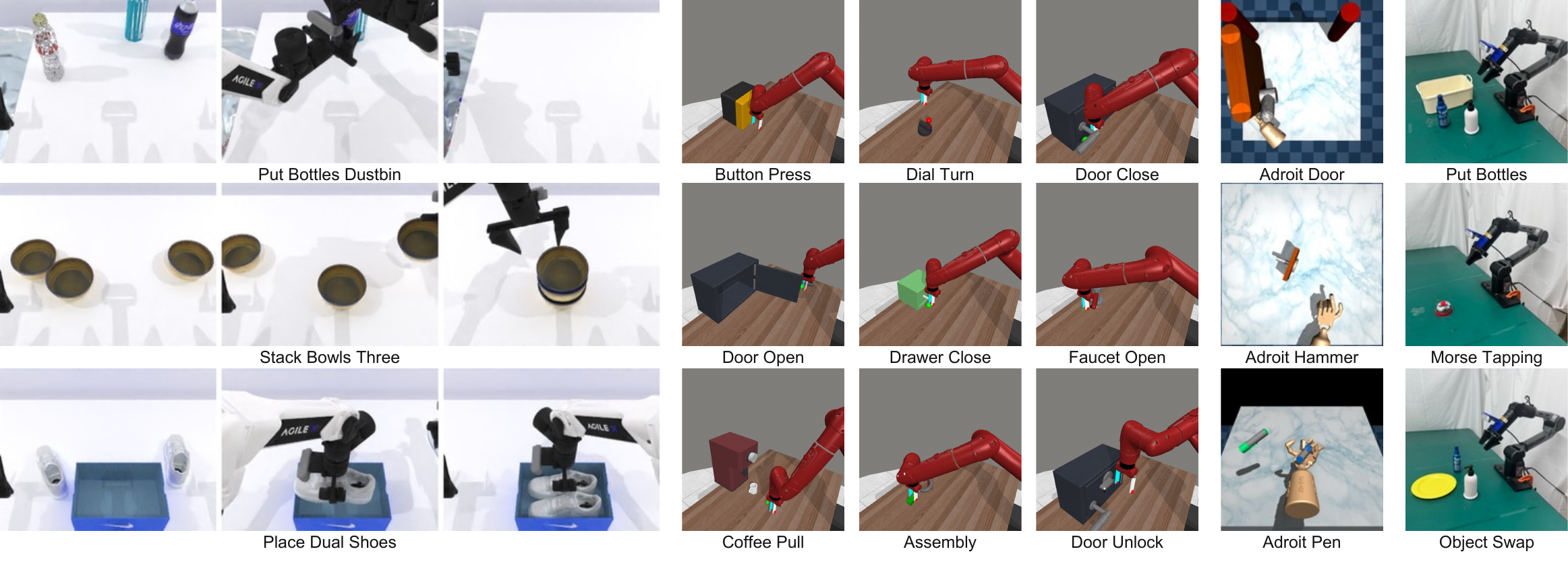}
    \caption{
    \textbf{Overview of Experimental Environments.} 
    The figure summarizes representative environments from both simulation and real-world experiments.
    In each row, the left three panels visualize the RoboTwin tasks, the middle three columns present representative MetaWorld tasks, the next column shows Adroit tasks, and the rightmost column shows our real-world tasks.
    }
    \label{fig:simulation_task_overview}
\end{figure*}

\subsection{Theoretical Analysis}
\label{subsec:theoretical analysis}
Here we provide a theoretical analysis of the benefits of history conditioning for imitation learning in partially observable environments.
Our goal is to characterize how the integration of temporal context mitigates the information loss inherent in POMDPs.
To analyze the impact of information sets on performance, we denote $\mathcal{L}_{\mathrm{diff}}(\pi; X_t)$ as the expected diffusion loss for a policy $\pi$ conditioned on variable $X_t$.
With the formalization of POMDP and observation aliasing in \cref{subsec:problem_setup}, and the imitation objective defined in \cref{eq:diffusion_loss}, we present two core propositions.
\begin{proposition}
\label{method:proposition1}
The minimum achievable diffusion-based imitation loss for a history-conditioned policy is always less than or equal to that of a reactive policy. That is,
\begin{equation}
    \min_{\theta} \mathcal{L}_{\mathrm{diff}}(\pi_\theta; h_t) \leq \min_{\theta} \mathcal{L}_{\mathrm{diff}}(\pi_\theta; o_t).
\end{equation}
\end{proposition}
\begin{proposition}
\label{method:proposition2}
In the presence of observation aliasing, history conditioning strictly reduces the minimum achievable imitation loss compared to a reactive policy:
\begin{equation}
    \min_{\theta} \mathcal{L}_{\mathrm{diff}}(\pi_\theta; h_t) < \min_{\theta} \mathcal{L}_{\mathrm{diff}}(\pi_\theta; o_t), \quad \text{whenever} \quad I_E(a_t; h_t \mid o_t) > 0.
\end{equation}
This condition implies that history resolves state ambiguity by capturing mutual information that is inaccessible to a reactive policy.
\end{proposition}
\Cref{method:proposition1} establishes that history conditioning never degrades the theoretical performance limit of the policy, while \cref{method:proposition2} proves that history conditioning strictly improves performance in the presence of observation aliasing.
These results demonstrate that history acts as a sufficient statistic for the belief state, allowing the policy to disambiguate latent environmental configurations through the capture of action-relevant mutual information. 
We summarize the high-level logic here and defer the complete mathematical derivations to the Appendix \ref{app:theoretical_analysis}.


\section{Experiments}
\subsection{Experimental Setup}
\label{subsec:exp_setup}
\begin{table*}[t]
\caption{
\textbf{Main Results on RoboTwin 2.0. }
Tasks are grouped into Short, Mid, and Long horizons. 
``Obs.'' denotes the observation modality, including RGB and point-cloud observations. 
\ours{} achieves the best overall performance with the smallest model size.
}
\centering
\small
\setlength{\tabcolsep}{5pt}

\begin{tabular}{llccccc}
\toprule
\multirow{2}{*}{\normalsize Method} 
& \multirow{2}{*}{\normalsize Obs.} 
& \multicolumn{4}{c}{\normalsize Success Rate by Horizon $\uparrow$} 
& \multirow{2}{*}{\normalsize Params $\downarrow$} \\
\cmidrule(lr){3-6}
& 
& Short (18) 
& Mid (15) 
& Long (17) 
& Average 
& \\
\midrule
ACT (RSS'23)    & RGB  & 29.94 & 26.40 & 32.47 & 29.74 & 80.0M \\
DP (RSS'23)       & RGB  & 31.28 & 26.00 & 26.41 & 28.04 & 96.8M \\
$\pi_0$(arXiv'24) & RGB  & 44.56 & 44.07 & 50.47 & 46.42 & 3.3B \\
RDT (ICLR'25)     & RGB  & 36.72 & 32.47 & 33.94 & 34.50 & 1.2B \\
SeedPolicy (arXiv'26) & RGB & 43.39 & 36.53 & 47.59 & 42.76 & 147.3M \\
DP3 (RSS'24)     & P.C. & 59.83 & 52.53 & 52.76 & 55.24 & 
264.4M \\
FlowPolicy (AAAI'25)  & P.C. & 54.89 & 37.40 & 29.59 & 	41.04 & 264.4M \\
MambaPolicy (IROS'25)
& P.C. & 61.67 & 56.27 & 58.88 & 59.10 & 47.9M \\

Diff-Control (IROS'24)
& P.C. & 64.00 & 54.67 & 58.82 & 59.44 & 264.4M \\

MTIL (RA-L'25)
& P.C. & $\mathbf{66.22}$ & 54.73 & 56.18 & 59.36 & 60.5M \\

CycleManip (CVPR'26)
& P.C. & 57.78 & 49.27 & 54.47 & 54.10 & 279.2M \\

\midrule
\ours{} (Ours) & P.C. 
& 64.78
& $\mathbf{57.33}$
& $\mathbf{64.06}$ 
& $\mathbf{62.30}$ 
& $\mathbf{44.3M}$ \\
\bottomrule
\end{tabular}

\label{tab:horizon_group_results}
\end{table*}
\textbf{Datasets.}
We evaluate \ours on diverse simulation benchmarks and real-world memory-dependent manipulation tasks, as shown in ~\Cref{fig:simulation_task_overview}.
For simulation, we use 87 tasks across three benchmarks: 50 bimanual tasks from RoboTwin 2.0~\citep{chen_robotwin_2025}, 34 single-arm tabletop tasks from MetaWorld~\citep{yu_meta-world_2021}, and 3 dexterous in-hand tasks from Adroit~\citep{rajeswaran_learning_2018}.
To analyze performance across task horizons, we group the 50 RoboTwin tasks into short-, mid-, and long-horizon categories based on average episode length.
Detailed grouping criteria, task lists, and subset selection are provided in Appendix~\ref{app:robotwin_horizon_grouping}. 

For real-world evaluation, we use an AgileX robotic platform equipped with a fixed Intel RealSense L515 camera for point-cloud observations.
We design three tabletop tasks that require long-horizon execution or progress tracking: 
(i) \textit{Put Bottles}, where the robot sequentially places two bottles into a basket;
(ii) \textit{Object Swap}, where the robot swaps two objects through an intermediate buffer slot;
and (iii) \textit{Morse Tapping}, where the robot taps a target three times before returning to the initial position.

\textbf{Baselines.}
We compare \ours{} with representative imitation-learning
baselines on each benchmark.
DP3~\citep{ze_3d_2024} is the most direct baseline, as
it uses the same 3D point-cloud observation and action
interfaces as \ours{}, while DP~\citep{chi_diffusion_2024}
serves as a standard diffusion-policy baseline.
On RoboTwin, we include ACT~\citep{zhao_learning_2023},
RDT~\citep{liu_rdt-1b_2025}, and
$\pi_0$~\citep{bai_openpi_2025} from the official benchmark
evaluation, along with SeedPolicy~\citep{gui_seedpolicy_2026},
FlowPolicy~\citep{zhang2025flowpolicy}, and
MambaPolicy~\citep{cao_mamba_2025}.
We further consider temporally informed baselines:
Diff-Control~\citep{liu2024diff}, which models action
statefulness using prior action trajectories;
MTIL~\citep{zhou_mtil_2025}, which encodes full trajectory
history with Mamba; and
CycleManip~\citep{wei2025cyclemanip}, which leverages
history sampling and progress prediction for
repetitive manipulation tasks.
On MetaWorld and Adroit, we compare with recent
policy-learning methods, including
AdaFlow~\citep{hu_adaflow_2024},
CP~\citep{prasad_consistency_2024}, and
MP1~\citep{sheng_mp1_2025}.


\textbf{Evaluation Metric and Protocol.}
We use task success rate as the primary metric.
For RoboTwin, we evaluate 100 test episodes under the in-distribution \textit{demo clean} setting, defining success as completion within the maximum horizon.
For MetaWorld and Adroit, we evaluate 20 episodes every 200 training epochs and report the average of the top five success rates.
For real-world experiments, each method undergoes 20 trials per task with randomized initial configurations based on task-specific criteria.


\textbf{Implementation Details.}
All models are trained with AdamW on NVIDIA RTX 4090 GPUs.
We use 50 demonstrations per task for RoboTwin and real-world experiments, and 10 demonstrations per task for MetaWorld and Adroit.
\ours{} employs trajectory-wise batching for causal history encoding. Training schedules, architectural details, and hyperparameters are specified in Appendix~\ref{app:implementation_details}.

\subsection{Simulation Experiments}

\textbf{Main Results.}
As shown in Table~\ref{tab:horizon_group_results},
\ours{} achieves the highest average success rate of
62.30\% on RoboTwin 2.0, outperforming the strongest
baseline, Diff-Control, by 4.8\% relatively and DP3
by 12.8\%.
On long-horizon tasks, \ours{} achieves 64.06\% success,
representing relative improvements of 8.8\% over
MambaPolicy and 21.4\% over DP3, demonstrating the
benefits of full-history conditioning.
Moreover, \ours{} achieves these results with the
smallest model size of 44.3M parameters.
On Adroit and MetaWorld, \ours{} also achieves the
highest overall average success rate among the
compared methods, as shown in
Table~\ref{tab:metaworld_adroit_results}.
\begin{table*}[t]
\caption{
\textbf{Results on Adroit and MetaWorld.} A comprehensive comparison
across 37 tasks using 3 random seeds.
Numbers in parentheses indicate the number of tasks in each benchmark.}
\centering
\small
\setlength{\tabcolsep}{5pt}
\begin{tabular}{lcccccc}
\toprule
\multirow{2}{*}{Method}
& \multirow{2}{*}{Adroit (3)}
& \multicolumn{4}{c}{MetaWorld}
& \multirow{2}{*}{Average} \\
\cmidrule(lr){3-6}
&
& Easy (21)
& Medium (4)
& Hard (4)
& Very Hard (5)
& \\
\midrule
DP (RSS'23)
& $21.0{\pm}7.7$
& $50.7{\pm}6.1$
& $11.0{\pm}2.5$
& $5.25{\pm}2.5$
& $22.0{\pm}5.0$
& $35.2{\pm}5.3$ \\

Adaflow (NeurIPS'24)
& $30.0{\pm}7.7$
& $49.4{\pm}6.8$
& $12.0{\pm}5.0$
& $5.75{\pm}4.0$
& $24.0{\pm}4.8$
& $35.6{\pm}6.1$ \\

CP (arXiv'24)
& $29.7{\pm}6.7$
& $69.3{\pm}4.2$
& $21.2{\pm}6.0$
& $17.5{\pm}3.9$
& $30.0{\pm}4.9$
& $50.1{\pm}4.7$ \\

DP3 (RSS'24)
& $67.3{\pm}5.0$
& $87.3{\pm}2.2$
& $44.5{\pm}8.7$
& $32.7{\pm}7.7$
& $39.4{\pm}9.0$
& $68.7{\pm}4.7$ \\


FlowPolicy (AAAI'25)
& $71.0{\pm}2.3$
& $84.8{\pm}2.2$
& $58.2{\pm}7.9$
& $40.2{\pm}4.5$
& $52.2{\pm}5.0$
& $71.6{\pm}3.5$ \\

MP1 (AAAI'26)
& $\mathbf{75.7{\pm}2.3}$
& $88.2{\pm}1.1$
& $\mathbf{68.0{\pm}3.1}$
& $\mathbf{58.1{\pm}5.0}$
& $67.2{\pm}2.7$
& $78.9{\pm}2.1$ \\

\midrule
\ours{} (ours)
& $73.0{\pm}2.9$
& $\mathbf{90.5{\pm}2.1}$
& $67.4{\pm}4.1$
& $54.6{\pm}2.4$
& $\mathbf{71.3{\pm}3.8}$
& $\mathbf{80.1{\pm}2.6}$ \\
\bottomrule
\end{tabular}
\label{tab:metaworld_adroit_results}
\end{table*}

\textbf{Analysis of History Encoding.}
To better understand how \ours uses historical information, we conduct controlled studies on the six-task long-horizon subset defined in 
Appendix~\ref{app:robotwin_horizon_grouping}.
We analyze three aspects: (1) how temporal backbone and history length affect performance, (2) how efficiently the history encoder scales to full-history conditioning, and (3) whether the policy truly relies on context representation when recent observations are corrupted.
Together, these studies show that the gains of \ours come from effective long-history utilization, while maintaining practical efficiency.

\begin{table*}[h]
\caption{Comparison of temporal backbones across history lengths.
We report average success over six long-horizon tasks and history-prefix encoding cost in p95 latency and peak GPU memory.}
\centering
\small
\setlength{\tabcolsep}{5pt}
\begin{tabular}{lccc}
\toprule
Encoder Backbone & History Length ($T_h$) & Success Rate (\%) $\uparrow$ & Encoding Cost $\downarrow$ \\
\midrule
\multirow{3}{*}{Transformer} 
    & 10 & 68.00 & 1.43 ms / 176.5 MB \\
    & 20 & 61.83 & 1.43 ms / 176.7 MB \\
    & Full history & 66.00 & 3.61 ms / 586.2 MB \\
\midrule
\multirow{3}{*}{Mamba} 
    & 10 & 53.00 & 1.87 ms / 181.1 MB \\
    & 20 & 60.00 & 1.90 ms / 181.5 MB \\
    & \textbf{Full history (Ours)} & \textbf{71.33} & \textbf{1.97 ms / 238.5 MB} \\
\bottomrule
\end{tabular}
\label{tab:history_length_comparison}
\end{table*}

We first compare the impact of history length on different encoder backbones (Table~\ref{tab:history_length_comparison}). While the Transformer encoder performs best with a short window ($T_h=10$) and stagnates with longer histories, the state-space encoder scales effectively with increasing context. Mamba achieves its peak success rate of 71.33\% using full history, which is an 8.1\% relative improvement over the full-history Transformer.
These results confirm that the recurrent formulation of Mamba is better suited for aggregating varying length observations into a compact context representation.

Beyond improving success rates, \ours also scales more efficiently to long histories. Full-history Transformer encoding costs 3.61 ms and a larger peak GPU memory (586.2 MB), whereas our full-history Mamba encoder requires only 1.97 ms and 238.5 MB, reducing latency by 45.4\% and peak memory by 59.3\%.
This advantage stems from the linear-time state-space formulation, which aggregates history through recurrent state updates instead of pairwise attention over all historical tokens.
Meanwhile, during streaming inference, \ours only maintains a compact hidden-state cache, making the per-step encoding overhead nearly independent of accumulated history length.


\begin{table*}[h]
\caption{
Robustness comparison under Gaussian perturbations to recent state representations. We report the six tasks' average success rate (\%) under different perturbation scales ($\sigma$) during inference.
}
\centering
\small
\setlength{\tabcolsep}{5pt}
\begin{tabular}{@{} l cccc @{}}
\toprule
\multirow{2}{*}{\textbf{Method}} & \multicolumn{4}{c}{\textbf{Noise Scale ($\sigma$)}} \\
\cmidrule(l){2-5}
& \textbf{0.00} (Clean) & \textbf{0.05} & \textbf{0.10} & \textbf{0.15} \\
\midrule
DP3 (Baseline) & 58.67 & 43.00 & 15.83 & 3.17 \\
DSSP (w/o full history) & 60.00 & 41.67 & 23.33 & 11.33 \\
\textbf{DSSP (Full history)} & \textbf{71.33} & \textbf{52.33} & \textbf{37.00} & \textbf{20.83} \\
\bottomrule
\end{tabular}
\label{tab:robustness_comparison}
\end{table*}
We further test whether the policy uses historical context when recent observations are unreliable by perturbing the most recent three state tokens during inference:
$z_i^{\mathrm{pert}} = z_i + \sigma \epsilon_i$, where $\epsilon_i \sim \mathcal{N}(0,\mathbf{I})$. As shown in Table~\ref{tab:robustness_comparison}, DP3 and our short-window ($T_h=10$) variant degrade sharply as the perturbation scale increases, while \ours{} with full history maintains substantially higher success rates. These results indicate that the learned context token incorporates earlier information, preventing the policy from degenerating into merely relying on recent observations for action generation.

To further assess the memory capabilities of
\ours{}, we evaluate it on memory-dependent
manipulation tasks from RMBench~\citep{chen2026rmbench},
with additional results provided in
Appendix~\ref{app:rmbench}.

\begin{table*}[h]
\caption{Ablation on six history-sensitive RoboTwin tasks under the \textit{demo clean} setting.
Hist., TD, Recent, Dyn., and Trans. denote the history encoder, timestep-decoupled denoising, recent-state conditioning, dynamics-aware loss, and Transformer denoising backbone, respectively.}
\centering
\small
\setlength{\tabcolsep}{5pt}
\begin{tabular}{lccccccc}
\toprule
Metric & DP3 & w/o Hist. & w/o TD & w/o Recent & w/o Dyn. & w/ Trans. & Full \\
\midrule
Success Rate (\%) & 58.67 & 64.33 & 66.67 & 68.17 & 69.50 & 68.50 & \textbf{71.33} \\
Relative Improvement (\%) & -- & +9.65 & +13.64 & +16.21 & +18.46 & +16.75 & \textbf{+21.56} \\
\bottomrule
\end{tabular}
\label{tab:ssm_dp_ablation}
\end{table*}

\textbf{Ablation Study.}
We ablate the key components of \ours{} on the six-task history-sensitive subset defined in Section~\ref{subsec:exp_setup}.
As shown in Table~\ref{tab:ssm_dp_ablation}, the full model achieves the best success rate of 71.33\%, a 21.56\% relative improvement over DP3.
Removing the history encoder leads to the largest drop, highlighting the importance of long-horizon context.
Replacing the Mamba action denoising backbone with a Transformer-based one remains stronger than DP3 but underperforms the full model, showing that Mamba further improves temporal action generation.
Without timestep-decoupled conditioning, performance drops to 66.67\%, validating the effectiveness of our decoupled conditioning design.
Finally, removing recent-state conditioning or the dynamics-aware loss also degrades performance, confirming the importance of local grounding and predictive context learning.

\subsection{Real-World Experiments}
\label{sec:real_world_exp}
We evaluate \ours on three real-world manipulation tasks that require long-horizon memory or progress tracking, as shown in Table~\ref{tab:real_world_results}.
\ours substantially outperforms DP3 across all tasks, increasing the average success rate from 30\% to 70\% (a 133.3\% relative improvement).
The gains are particularly significant on Morse Tapping, where \ours improves the success rate from 15\% to 85\%.
These results demonstrate the effectiveness of full-history context for real-world memory-dependent manipulation, where the decision making often depends on previous interactions rather than the current observation alone.
To better understand where the gains come from, we provide a failure-mode analysis for each real-world task in Appendix~\ref{app:real_world_failure_analysis}.

\begin{table*}[!h]
\caption{
Real-world long-horizon task (avg. steps) evaluation.
}
\centering
\small
\setlength{\tabcolsep}{4.5pt}
\begin{tabular}{lccc}
\toprule
Method & Put Bottles (666) & Object Swap (713) & Morse Tapping (366) \\
\midrule
FlowPolicy (AAAI'25) & 40\% & 30\% & 10\% \\
MP1 (AAAI'26) & 45\% & 30\% & 25\% \\
DP3 (RSS'24) & 40\% & 35\% & 15\% \\
\ours{} (ours) & \textbf{60\%} & \textbf{65\%} & \textbf{85\%} \\
\bottomrule
\end{tabular}

\label{tab:real_world_results}
\end{table*}

\section{Conclusion}
\label{conclusion}
In this paper, we introduce \ours{}, an efficient full-history conditioned diffusion state-space policy for long-horizon robot manipulation.
Our results show that compactly encoding the full observation history improves temporal disambiguation and task-progress tracking in history-sensitive tasks.
The proposed state-space history encoder, dynamics-aware objective, and hierarchical timestep-decoupled conditioning jointly integrate long-horizon context with recent observations for action generation.

\section{Limitations}
\label{limitaions}
DSSP is primarily designed to improve temporal disambiguation and task-progress tracking through full-history conditioning. It does not directly eliminate low-level perception or control failures, such as inaccurate target localization, grasping, or placement, which remain a source of real-world failure. Moreover, our real-world evaluation is limited to three tabletop tasks with a fixed-camera point-cloud setup. Extending DSSP to more diverse embodiments, viewpoints, deformable objects, and dynamic environments remains an important direction for future work.
\begin{ack}
This work has been supported by the National Natural Science Foundation of China (No. 62503322), the AI for Science Seed Program of Shanghai Jiao Tong University (Project No. 2025AI4SQY06), the Shanghai Jiao Tong University Xiaomi Scholar Fund, and the Shanghai Municipal Special Program for Basic Research on General AI Foundation Models (Grant No. 2025SHZDZX026D18).
\end{ack}

\bibliography{main, longhorizon,benchmark,diffusionparadigm,mamba,vla,citation}
\bibliographystyle{plainnat}


\appendix
\section{Preliminaries}
\label{app:preliminaries}
\textbf{Diffusion Policy.}
Diffusion Policy~\citep{chi_diffusion_2024} adapts Denoising Diffusion Probabilistic Models (DDPMs)~\citep{ho2020denoisingdiffusionprobabilisticmodels} to action generation. The policy treats a future action sequence $a_{t:t+H_a-1}$ of horizon $H_a$ as the clean sample $x_0$. We parameterize the model to predict the clean action sequence directly, conditioning the denoising process on the history representation $h_t$. During training, we optimize the reconstruction objective over $L$ diffusion steps:
\begin{equation}
\mathcal{L}_{x_0} = \mathbb{E}_{x_0,\tau,\epsilon} \left[ \left\| f_\theta(x_\tau,\tau,h_t) - x_0 \right\|_2^2 \right],
\end{equation}
where $\tau \in \{1, \dots, L\}$ is the uniform diffusion step and $x_\tau$ is the noise-corrupted action sequence.

\textbf{State Space Models (SSMs).}
We employ Mamba as the sequence backbone for both history encoding and action diffusion. A standard discrete SSM updates a hidden state $s_t \in \mathbb{R}^N$ and outputs $y_t$ via time-invariant parameters:
\begin{equation}
s_t = \bar{A} s_{t-1} + \bar{B} x_t, \qquad y_t = C s_t.
\end{equation}
Mamba introduces a selective mechanism where parameters become input-dependent: $(\Delta_t, B_t, C_t) = f_\theta(x_t)$, with $\Delta_t$ serving as the discretization step size. The discretized parameters are:
\begin{equation}
\bar{A}_t = \exp(\Delta_t A), \qquad \bar{B}_t = (\Delta_t A)^{-1}\big(\exp(\Delta_t A)-I\big)\Delta_t B_t.
\end{equation}
The recurrent update becomes $s_t = \bar{A}_t s_{t-1} + \bar{B}_t x_t$ and $y_t = C_t s_t$. This input-dependent selection enables the model to efficiently compress long observation histories and generate temporally coherent actions.

\section{Additional Related Works}

\subsection{Vision-Language-Action Models}

Vision-language-action (VLA) models have recently become a central direction for scalable robot learning.
RT-2~\citep{brohan2023rt2visionlanguageactionmodelstransfer} transfers vision-language knowledge to robotic control by representing actions as language-like tokens.
Octo~\citep{octomodelteam2024octoopensourcegeneralistrobot} and OpenVLA~\citep{kim_openvla_2024} further introduce open-source generalist policies trained on large cross-embodiment robot datasets.
Recent generative VLA models move beyond discrete action tokenization toward continuous action generation, including flow-matching-based policies such as $\pi_0$~\citep{bai_openpi_2025} and large-scale diffusion policies such as RDT-1B~\citep{liu_rdt-1b_2025}.

Recent work has further improved VLA policies along several directions.
OpenVLA-OFT~\citep{kim2025finetuningvisionlanguageactionmodelsoptimizing} studies effective fine-tuning recipes with continuous actions and action chunking, while VITA-VLA~\citep{gao2026vitavisiontoactionflowmatching} equips pretrained vision-language models with action-generation capability through action expert distillation.
Other works improve inference efficiency and temporal consistency through asynchronous generation, coarse-to-fine action generation, or action coherence guidance~\citep{jiang2025asyncvlaasynchronousflowmatching,tang2025vlashrealtimevlasfuturestateaware}.
In parallel, SpatialVLA~\citep{qu2025spatialvlaexploringspatialrepresentations} and GraphCoT-VLA~\citep{huang2025graphcotvla3dspatialawarereasoning} enhance spatial reasoning and embodiment-aware manipulation.

\subsection{World Models for Robotics}

World models provide a complementary direction for robot learning by predicting future environment states for simulation, planning, or policy improvement.
In robotic manipulation, action-conditioned video generation has been widely used to model robot-object dynamics.
IRASim~\citep{zhu2025irasimfinegrainedworldmodel} learns an interactive real-robot action simulator conditioned on observations and robot actions.
RoboMaster~\citep{fu2026learningvideogenerationrobotic} improves trajectory-controlled robotic video generation by modeling robot-object interactions, while Ctrl-World~\citep{guo2026ctrlworldcontrollablegenerativeworld} studies controllable multi-view world modeling for evaluating and improving generalist robot policies.

Recent works further introduce structured spatial representations for world modeling.
FlowDreamer~\citep{guo2025flowdreamerrgbdworldmodel} uses RGB-D observations and 3D scene flow for action-conditioned future prediction.
GAF~\citep{chai2025gafgaussianactionfield} and GWM~\citep{lu2025gwmscalablegaussianworld} represent dynamic manipulation scenes with Gaussian-based formulations, enabling future scene prediction and action refinement.
Dream2Flow~\citep{dharmarajan2025dream2flowbridgingvideogeneration} connects video generation with robot control by extracting 3D object flow from generated videos.
World4RL~\citep{jiang2026world4rldiffusionworldmodels} and PlayWorld~\citep{yin2026playworldlearningrobotworld} further use learned world models for policy evaluation, reinforcement learning, or policy refinement.

\section{Implementation Details}
\label{app:implementation_details}

\subsection{Trajectory-Wise Training}
\label{app:trajectory_wise_training}
Since \ours{} conditions action generation on causal history context, we use trajectory-wise batch construction during training.
At each optimization step, we randomly load one complete demonstration trajectory and compute its state representations and causal context representations following \Cref{eq:observation_encoder,eq:history_encoding}.
We then sample $B$ valid action-window start indices from this trajectory, where $B$ is the training batch size.

For each sampled start time $t_i$, we construct the hierarchical condition using \Cref{eqs:prefix_condition}, with the context representations and recent state representations available up to $t_i$.
The corresponding diffusion target is the future action chunk starting from $t_i$.
This ensures that every training sample is conditioned only on observations before its prediction time, without accessing future observations.

For fair comparison with window-based baselines, we keep the same effective batch size $B$ and comparable optimization budget across methods.
Thus, trajectory-wise training does not introduce additional demonstrations or a larger denoising batch; it only changes how causal history conditions are constructed for \ours{}.
The training objectives follow \Cref{eq:total_loss}.

\begin{table}[h]
\centering
\small
\setlength{\tabcolsep}{8pt}
\caption{Hyperparameters used for RoboTwin simulation experiments.}
\label{tab:robotwin_hyperparams}
\begin{tabular}{lc}
\toprule
Hyperparameter & Value \\
\midrule
Recent Observation Horizon ($N$) & 3 \\
Horizon ($H$) & 8 \\
Action Execution Horizon ($H_a$) & 6 \\
Expert Demonstrations per Task & 50 \\
Optimizer & AdamW \\
Betas $(\beta_1,\beta_2)$ & $(0.95, 0.999)$ \\
Learning Rate & $1.0\times 10^{-4}$ \\
Weight Decay & $1.0\times 10^{-6}$ \\
Diffusion Training Steps & 100 \\
Inference Steps & 10 \\
Learning Rate Scheduler & Cosine \\
Warmup Steps & 500 \\
Prediction Type & Sample prediction \\
Dynamics Loss Weight ($\lambda$) & 0.05 \\
Dynamics Loss Type & Cosine distance \\
Mamba Backbone Layers & 8 \\
Mamba History Encoder Layers & 2 \\
Hidden Dimension & 512 \\
SSM State Dimension & 64 \\
\bottomrule
\end{tabular}
\end{table}

\begin{table}[h]
\centering
\small
\setlength{\tabcolsep}{8pt}
\caption{Hyperparameters used for Adroit simulation experiments.}
\label{tab:adroit_hyperparams}
\begin{tabular}{lc}
\toprule
Hyperparameter & Value \\
\midrule
Recent Observation Horizon ($N$) & 2 \\
Horizon ($H$) & 4 \\
Action Execution Horizon ($H_a$) & 3 \\
Expert Demonstrations per Task & 10 \\
Optimizer & AdamW \\
Betas $(\beta_1,\beta_2)$ & $(0.95, 0.999)$ \\
Learning Rate & $1.0\times 10^{-4}$ \\
Weight Decay & $1.0\times 10^{-6}$ \\
Epoch & 3000 \\
Learning Rate Scheduler & Cosine \\
Warmup Steps & 500 \\
Diffusion Training Steps & 100 \\
Inference Steps & 10 \\
Prediction Type & Sample prediction \\
Dynamics Loss Weight ($\lambda$) & 0.05  \\
Dynamics Loss Type & Cosine distance \\
Mamba Backbone Layers & 8 \\
Mamba History Encoder Layers & 2 \\
Hidden Dimension & 512 \\
SSM State Dimension & 64 \\
Evaluation Episodes & 20 \\
Maximum Evaluation Steps & 300 \\
\bottomrule
\end{tabular}
\end{table}

\begin{table}[h]
\centering
\small
\setlength{\tabcolsep}{8pt}
\caption{Hyperparameters used for MetaWorld simulation experiments.}
\label{tab:metaworld_hyperparams}
\begin{tabular}{lc}
\toprule
Hyperparameter & Value \\
\midrule
Recent Observation Horizon ($N$) & 2 \\
Horizon ($H$) & 4 \\
Action Execution Horizon ($H_a$) & 3 \\
Expert Demonstrations per Task & 10 \\
Optimizer & AdamW \\
Betas $(\beta_1,\beta_2)$ & $(0.95, 0.999)$ \\
Learning Rate & $1.0\times 10^{-4}$ \\
Weight Decay & $1.0\times 10^{-6}$ \\
Epoch & 1000 \\
Learning Rate Scheduler & Cosine \\
Warmup Steps & 500 \\
Diffusion Training Steps & 100 \\
Inference Steps & 10 \\
Prediction Type & Sample prediction \\
Dynamics Loss Weight ($\lambda$) & 0.05 \\
Dynamics Loss Type & Cosine distance \\
Mamba Backbone Layers & 8 \\
Mamba History Encoder Layers & 2 \\
Hidden Dimension & 512 \\
SSM State Dimension & 64 \\
Evaluation Episodes & 20 \\
Maximum Evaluation Steps & 1000 \\
\bottomrule
\end{tabular}
\end{table}

\subsection{Hyperparameters}
\label{app:robotwin_hyperparameters}

To account for the varying difficulty levels and unique characteristics of different benchmarks, we tailor our hyperparameter configurations to each individual dataset. The final settings, which include additional configurations for the Mamba-based history encoder, Mamba diffusion backbone, and dynamics-aware auxiliary objective and are summarized in Tables \ref{tab:robotwin_hyperparams},\ref{tab:adroit_hyperparams},\ref{tab:metaworld_hyperparams}, are informed by established practices in prior literature \citep{ze_3d_2024,zhang2025flowpolicy,chen_robotwin_2025}. 

\begin{table}[h]
\centering
\small
\setlength{\tabcolsep}{5pt}
\renewcommand{\arraystretch}{1.25}
\caption{RoboTwin task grouping by average episode length. Short, mid, and long horizons correspond to average episode lengths of $<150$, $150$--$250$, and $>250$ steps, respectively.}
\label{tab:robotwin_horizon_groups}
\begin{tabular}{p{0.16\linewidth}p{0.78\linewidth}}
\toprule
Horizon Group & Tasks \\
\midrule
Long (17) &
Put Bottles Dustbin, Open Microwave, Stack Blocks Three, Stack Bowls Three, Blocks Ranking Rgb, Blocks Ranking Size, Hanging Mug, Stack Blocks Two, Stack Bowls Two, Place Cans Plasticbox, Handover Block, Shake Bottle Horizontally, Put Object Cabinet, Dump Bin Bigbin, Open Laptop, Place Can Basket, Place Object Basket \\
\midrule
Mid (15) &
Shake Bottle, Place Burger Fries, Place Bread Basket, Place Dual Shoes, Handover Mic, Place Shoe, Place Empty Cup, Scan Object, Place Bread Skillet, Place Container Plate, Place A2B Left, Rotate Qrcode, Move Stapler Pad, Stamp Seal, Move Can Pot \\
\midrule
Short (18) &
Place Mouse Pad, Place Fan, Adjust Bottle, Move Pillbottle Pad, Place Object Scale, Place A2B Right, Press Stapler, Place Object Stand, Place Phone Stand, Pick Dual Bottles, Pick Diverse Bottles, Move Playingcard Away, Beat Block Hammer, Lift Pot, Turn Switch, Grab Roller, Click Bell, Click Alarmclock \\
\bottomrule
\end{tabular}
\end{table}
\subsection{RoboTwin Horizon Grouping}
\label{app:robotwin_horizon_grouping}

For horizon-wise analysis, we partition the 50 RoboTwin tasks according to their average episode length. Tasks with average length below 150 steps are categorized as short-horizon tasks, tasks between 150 and 250 steps are categorized as mid-horizon tasks, and tasks above 250 steps are categorized as long-horizon tasks. The resulting groups are listed in Table~\ref{tab:robotwin_horizon_groups}.

For ablation and diagnostic analysis, we further define a six-task long-horizon analysis subset from the above grouping.
This subset is selected from tasks with long execution horizons and temporally dependent manipulation behaviors.
The purpose of this subset is not to replace the full benchmark evaluation, but to provide a controlled set of tasks for studying how different architectural components affect history utilization.
Unless otherwise specified, all ablation studies and history-utilization analyses are conducted on this subset.
The selected tasks and their task-level results are shown in 
Table~\ref{tab:long_horizon_subset_details}.
\begin{table*}[h]
\caption{Six-task history-sensitive analysis subset used for ablation and diagnostic studies. 
The table reports task-level success rates under the in-distribution (\textit{demo clean}) evaluation setting.}
\centering
\small
\setlength{\tabcolsep}{6pt}
\begin{tabular}{lccc}
\toprule
Task & Avg. Steps & DP3 (\%) & \ours (\%) \\
\midrule
Put Bottles Dustbin & 637 & 60 & 83 \\
Stack Bowls Three & 476 & 57 & 80 \\
Put Object Cabinet & 274 & 72 & 70 \\
Place Dual Shoes & 228 & 13 & 19 \\
Stack Bowls Two & 313 & 83 & 93 \\
Place Can Basket & 255 & 67 & 83 \\
\midrule
Average & 364 & 58.67 & 71.33 \\
\bottomrule
\end{tabular}
\label{tab:long_horizon_subset_details}
\end{table*}
\section{Additional Results and Failure-Mode Analysis}
\label{app:additional_results_failure_analysis}

\subsection{Per-task RoboTwin results.}
\Cref{tab:appendix_robotwin_dp3_vs_ours_50task} reports the per-task success rates of DP3 and \ours{} on all 50 RoboTwin tasks under the clean setting.
Overall, \ours{} improves the average success rate from 55.24\% to 62.30\%, with particularly large gains on tasks requiring sequential progress tracking, such as Open Microwave, Place Cans Plasticbox, Put Bottles Dustbin, and Stack Bowls Three.
\begin{longtable}{lrr}
\caption{Per-task success rate comparison between RoboTwin DP3 and \ours{} on 50 RoboTwin tasks under the clean setting.}
\label{tab:appendix_robotwin_dp3_vs_ours_50task}\\
\toprule
Task & DP3 & \ours{} \\
\midrule
\endfirsthead
\toprule
Task & DP3 & \ours{} \\
\midrule
\endhead
Adjust Bottle & \textbf{99.00} & 96.00 \\
Beat Block Hammer & 72.00 & \textbf{79.00} \\
Blocks Ranking RGB & 3.00 & \textbf{6.00} \\
Blocks Ranking Size & 2.00 & \textbf{4.00} \\
Click Alarmclock & 77.00 & \textbf{99.00} \\
Click Bell & 90.00 & \textbf{100.00} \\
Dump Bin Bigbin & \textbf{85.00} & 84.00 \\
Grab Roller & \textbf{98.00} & \textbf{98.00} \\
Handover Block & 70.00 & \textbf{95.00} \\
Handover Mic & \textbf{100.00} & 93.00 \\
Hanging Mug & 17.00 & \textbf{24.00} \\
Lift Pot & \textbf{97.00} & 96.00 \\
Move Can Pot & 70.00 & \textbf{86.00} \\
Move Pillbottle Pad & 41.00 & \textbf{58.00} \\
Move Playingcard Away & 68.00 & \textbf{71.00} \\
Move Stapler Pad & 12.00 & \textbf{16.00} \\
Open Laptop & 82.00 & \textbf{88.00} \\
Open Microwave & 61.00 & \textbf{97.00} \\
Pick Diverse Bottles & 52.00 & \textbf{53.00} \\
Pick Dual Bottles & 60.00 & \textbf{66.00} \\
Place A2B Left & \textbf{46.00} & 40.00 \\
Place A2B Right & 49.00 & \textbf{52.00} \\
Place Bread Basket & 26.00 & \textbf{29.00} \\
Place Bread Skillet & 19.00 & \textbf{39.00} \\
Place Burger Fries & 72.00 & \textbf{81.00} \\
Place Can Basket & 67.00 & \textbf{83.00} \\
Place Cans Plasticbox & 48.00 & \textbf{88.00} \\
Place Container Plate & 86.00 & \textbf{95.00} \\
Place Dual Shoes & 13.00 & \textbf{19.00} \\
Place Empty Cup & 65.00 & \textbf{86.00} \\
Place Fan & 36.00 & \textbf{40.00} \\
Place Mouse Pad & \textbf{4.00} & \textbf{4.00} \\
Place Object Basket & \textbf{65.00} & 62.00 \\
Place Object Scale & \textbf{15.00} & 10.00 \\
Place Object Stand & 60.00 & \textbf{61.00} \\
Place Phone Stand & 44.00 & \textbf{60.00} \\
Place Shoe & \textbf{58.00} & 49.00 \\
Press Stapler & \textbf{69.00} & 66.00 \\
Put Bottles Dustbin & 60.00 & \textbf{83.00} \\
Put Object Cabinet & \textbf{72.00} & 70.00 \\
Rotate QRcode & \textbf{74.00} & 66.00 
\\
\midrule
Scan Object & \textbf{31.00} & 29.00 \\
Shake Bottle & 98.00 & \textbf{100.00} \\
Shake Bottle Horizontally & \textbf{100.00} & \textbf{100.00} \\
Stack Blocks Three & 1.00 & \textbf{2.00} \\
Stack Blocks Two & 24.00 & \textbf{30.00} \\
Stack Bowls Three & 57.00 & \textbf{80.00} \\
Stack Bowls Two & 83.00 & \textbf{93.00} \\
Stamp Seal & 18.00 & \textbf{32.00} \\
Turn Switch & 46.00 & \textbf{57.00} \\
\midrule
\textbf{Average} & 55.24 & \textbf{62.30} \\
\bottomrule
\end{longtable}

\subsection{Memory-Dependent Evaluation on RMBench}
\label{app:rmbench}

To further evaluate the memory capabilities of \ours{},
we conduct additional experiments on RMBench~\citep{chen2026rmbench},
a benchmark specifically designed for memory-dependent
robotic manipulation.
We select three representative tasks with distinct
memory requirements:
\textit{Put Back Block}, which requires remembering
an object's initial position;
\textit{Swap T}, which requires retaining historical
position and orientation information;
and \textit{Blocks Ranking Try}, which requires
tracking previous attempts and their outcomes.

We compare \ours{} with imitation-learning
baselines, including DP\citep{chi_diffusion_2024}, MambaPolicy\citep{cao_mamba_2025}, Diff-Control\citep{liu2024diff},
MTIL\citep{zhou_mtil_2025}, and CycleManip\citep{wei2025cyclemanip}, as well as memory-oriented VLA baselines, including $\pi_{0.5}$\citep{intelligence__05_2025}, Mem-0\citep{chen2026rmbench},
MemoryVLA\citep{shi2026memoryvlaperceptualcognitivememoryvisionlanguageaction}, and EventVLA\citep{yang2026eventvlaeventdrivenvisualevidence}.
Our policy uses RGB observations from three camera
views, consisting of one main camera and two
wrist-mounted cameras.

\begin{table*}[hb]
\caption{
\textbf{Results on memory-dependent RMBench tasks.}
Success rates (\%) across three tasks requiring
different forms of historical memory.
The best result in each column is highlighted in bold.
}
\centering
\small
\setlength{\tabcolsep}{7pt}
\begin{tabular}{lcccc}
\toprule
Method
& Put Back Block
& Swap T
& Blocks Ranking Try
& Average \\
\midrule

DP
& 0 & 2 & 0 & 0.67 \\

MambaPolicy
& 0 & 2 & 12 & 4.67 \\

Diff-Control
& 0 & 14 & 9 & 7.67 \\

MTIL
& 37 & 35 & 41 & 37.67 \\

CycleManip
& 0 & 88 & $\mathbf{94}$ & 60.67 \\

\midrule

$\pi_{0.5}$
& 11 & 15 & 6 & 10.67 \\

Mem-0
& 90 & 14 & 18 & 40.67 \\

MemoryVLA (OpenVLA)
& 50 & 9 & 12 & 23.67 \\

MemoryVLA (QwenOFT)
& 81 & 9 & 53 & 47.67 \\

EventVLA (visual anchors only)
& 95 & 87 & 81 & 87.67 \\

\midrule

\ours{} (Ours)
& $\mathbf{100}$
& $\mathbf{93}$
& 87
& $\mathbf{93.33}$ \\

\bottomrule
\end{tabular}
\label{tab:rmbench_results}
\end{table*}

As shown in Table~\ref{tab:rmbench_results},
\ours{} achieves the highest average success rate
of 93.33\%, outperforming the strongest competing
baseline, EventVLA (87.67\%), by 5.66 percentage
points.
In particular, \ours{} achieves 100\% on
Put Back Block and 93\% on Swap T, demonstrating
its ability to retain task-relevant spatial information.
It also achieves 87\% on Blocks Ranking Try,
indicating its capability to utilize historical
feedback across multiple attempts.
These results further demonstrate the effectiveness
of full-history encoding for memory-dependent
manipulation.

\subsection{Effect of History Latent Dimension}
\label{app:latent_dim_ablation}

We investigate the effect of history latent
dimensionality on the six-task RoboTwin analysis
subset, keeping all other training and evaluation
settings unchanged.
As shown in Table~\ref{tab:latent_dim_ablation},
reducing the dimension from 512 to 256 decreases
the average success rate from 71.33\% to 66.67\%,
while increasing it to 1024 yields 69.83\%.
These results suggest that the default dimension
of 512 provides sufficient representation capacity
for the evaluated tasks without introducing
unnecessary model complexity.

\begin{table*}[h]
\caption{
\textbf{Effect of history latent dimension.}
Average success rates (\%) across six RoboTwin
analysis tasks.
}
\centering
\small
\setlength{\tabcolsep}{10pt}
\begin{tabular}{lccc}
\toprule
Latent Dimension & 256 & 512 & 1024 \\
\midrule
Success Rate (\%) & 66.67 & \textbf{71.33} & 69.83 \\
\bottomrule
\end{tabular}
\label{tab:latent_dim_ablation}
\end{table*}

\subsection{Dynamics Predictor Analysis}
\label{app:dynamics_analysis}

To assess whether the dynamics-aware objective
encourages the learned context to capture
action-dependent state transitions, we conduct
additional evaluations on the six-task RoboTwin
analysis subset.

Specifically, we compare predictions conditioned
on the executed action with those conditioned on
a temporally mismatched action sampled from
another timestep of the same rollout, while
keeping the history context fixed.
We evaluate prediction quality using two metrics:
(1) \textit{Transition Cosine}, which measures the
alignment between the predicted latent transition
$\hat{z}_{t+1}-z_t$ and the actual transition
$z_{t+1}-z_t$;
and (2) \textit{Decoded Joint-state RMSE}, which
measures the error in predicting 12 arm-joint
positions through an evaluation-only linear probe
trained on offline demonstrations and frozen
during evaluation.

\begin{table}[h]
\caption{
\textbf{Dynamics predictor evaluation.}
Comparison of state transition prediction
under executed and temporally mismatched actions.
}
\centering
\small
\setlength{\tabcolsep}{6pt}
\begin{tabular}{lcc}
\toprule
Metric
& Correct Action
& Mismatched Action \\
\midrule
Transition Cosine $\uparrow$
& $\mathbf{0.108}$ & $0.056$ \\
Joint-state RMSE (rad) $\downarrow$
& $\mathbf{0.036}$ & $0.201$ \\
\bottomrule
\end{tabular}
\label{tab:dynamics_diagnostic}
\end{table}

As shown in Table~\ref{tab:dynamics_diagnostic},
conditioning on the executed action yields
higher transition alignment and substantially
lower joint-state prediction error than using
a mismatched action.
These results indicate that the dynamics-aware
predictor captures action-dependent transition
information, supporting its role in learning
predictive history representations.

\subsection{Failure-Mode Analysis on Real-World Tasks}
\label{app:real_world_failure_analysis}

We further analyze the failure modes of each real-world task to better understand where the gains of \ours{} come from.

\noindent\textbf{\textit{Morse Tapping.}}
\ours{} achieves the largest relative improvement on this task.
The policy is required to tap the target three times before returning, which demands accurate tracking of the number of completed taps.
However, the visual observation before each tap is nearly identical, making the task progress ambiguous when only short-term context is available.
As a result, DP3 often stops at an incorrect stage or performs redundant taps.
By maintaining a history context, \ours{} better tracks the tapping progress and executes the correct number of taps.
The remaining failures of \ours{} mainly come from inaccurate target localization, which can cause missed or imprecise contacts.

\noindent\textbf{\textit{Put Bottles.}}
This long-horizon task requires the robot to sequentially place multiple bottles into the basket.
A common failure mode of DP3 is premature termination after placing only one bottle.
This happens because intermediate states, where one bottle has already been placed, can appear visually similar to the final completion state when the policy only observes a short recent context.
With historical context, \ours{} better infers the overall task progress and distinguishes intermediate subgoals from true task completion.
Its failures are more often caused by local manipulation errors, such as inaccurate grasping or placement, rather than losing track of the task stage.

\noindent\textbf{\textit{Object Swap.}}
This task requires swapping two objects through an intermediate buffer slot, with demonstrations collected in both swap directions.
When an object is located in the buffer, the current observation alone is insufficient to determine whether it should be moved to the left or to the right.
Consequently, DP3 may move the object back to its original location, leading to a reversal of progress.
In contrast, \ours{} uses the historical context to infer the previous movement direction and resolve this ambiguity.
This allows the policy to maintain consistent task progress across visually aliased intermediate states.

\section{Theoretical Analysis}
\label{app:theoretical_analysis}
We provide the formal proofs for the propositions presented in the main text regarding the theoretical advantages of history conditioning in imitation learning. 
We analyze the diffusion noise-prediction loss directly.  This distinction is important:
its Bayes-optimal predictor is a conditional denoiser, not the conditional mean of the
clean expert action.  As a result, the loss is sensitive to the entire conditional action
distribution, including changes that preserve both its mean and variance.

Let $A:=A_t\in\mathbb{R}^d$, $H:=H_t$, and $O:=O_t$, with $H$ and $O$
taking values in standard Borel spaces so that regular conditional laws exist.  Let $K$
be a diffusion step sampled from a finite schedule, independently of $(A,H,O)$, and let
$\epsilon\sim\mathcal{N}(0,I_d)$ be independent noise.  At step $k$, the noised action is
\begin{equation}
    Z_k = \alpha_k A + \sigma_k\epsilon.
    \label{eq:forward_diffusion}
\end{equation}
We assume $\mathbb{E}\lVert A\rVert^2<\infty$ and that at least one step $k$ with
$P(K=k)>0$ is nondegenerate, i.e., $\alpha_k\neq 0$ and $\sigma_k>0$.  Standard
diffusion schedules satisfy this condition.  For $X\in\{O,H\}$ and a measurable
denoiser $f$ with finite second moment, define
\begin{equation}
    \mathcal{L}_{\mathrm{diff}}(f;X)
    := \mathbb{E}\!\left[
        \left\lVert\epsilon-f(Z_K,K,X)\right\rVert^2
    \right],
    \qquad
    \mathcal{L}^*(X):=\inf_f\mathcal{L}_{\mathrm{diff}}(f;X).
    \label{eq:population_diffusion_loss}
\end{equation}
Here the infimum is over all such measurable denoisers, so $\mathcal{L}^*$ is a
population Bayes risk.  If the parameterized denoiser class realizes the corresponding
Bayes predictors, then $\mathcal{L}^*(X)=\inf_\theta
\mathcal{L}_{\mathrm{diff}}(\pi_\theta;X)$.  For restricted neural classes, the
non-strict result below still follows when the history-conditioned class contains the
reactive class, whereas strictness additionally requires sufficient realizability.

\subsection{History cannot increase the optimal population diffusion loss.}
\begin{proposition}
\label{proposition1}
The optimal population diffusion loss satisfies
\begin{equation}
    \mathcal{L}^*(H)\leq \mathcal{L}^*(O).
\end{equation}
\end{proposition}

\begin{proof}
We first identify the optimal denoiser for an arbitrary conditioning variable $X$.
Let
\begin{equation}
    f_X^*(Z_K,K,X):=\mathbb{E}[\epsilon\mid Z_K,K,X].
    \label{eq:bayes_denoiser}
\end{equation}
For any measurable denoiser $f(Z_K,K,X)$ with finite second moment, add and subtract
$f_X^*$ inside the prediction error:
\begin{equation}
    \epsilon-f
    = (\epsilon-f_X^*)+(f_X^*-f).
\end{equation}
Expanding the squared norm and taking expectations gives
\begin{align}
    \mathcal{L}_{\mathrm{diff}}(f;X)
    &= \mathbb{E}\lVert\epsilon-f_X^*\rVert^2
       +\mathbb{E}\lVert f_X^*-f\rVert^2 \notag\\
    &\quad
       +2\mathbb{E}\!\left[(\epsilon-f_X^*)^\top(f_X^*-f)\right].
    \label{eq:generic_risk_expansion}
\end{align}
The last term is zero.  Indeed, $f_X^*-f$ is measurable with respect to
$\sigma(Z_K,K,X)$, whereas the defining property of conditional expectation gives
\begin{equation}
    \mathbb{E}[\epsilon-f_X^*\mid Z_K,K,X]=0.
\end{equation}
Thus, by the tower property,
\begin{align}
    &\mathbb{E}\!\left[(\epsilon-f_X^*)^\top(f_X^*-f)\right] \notag\\
    &\quad=\mathbb{E}\!\left[
        \mathbb{E}\!\left[
            (\epsilon-f_X^*)^\top(f_X^*-f)
            \mid Z_K,K,X
        \right]
    \right] \notag\\
    &\quad=\mathbb{E}\!\left[
        \mathbb{E}[\epsilon-f_X^*\mid Z_K,K,X]^\top
        (f_X^*-f)
    \right]=0.
    \label{eq:generic_cross_term_zero}
\end{align}
Substituting this equality into \eqref{eq:generic_risk_expansion} yields
\begin{equation}
    \mathcal{L}_{\mathrm{diff}}(f;X)
    =\mathbb{E}\lVert\epsilon-f_X^*\rVert^2
     +\mathbb{E}\lVert f_X^*-f\rVert^2.
    \label{eq:conditional_expectation_projection}
\end{equation}
The second term is nonnegative and vanishes for $f=f_X^*$.  Therefore $f_X^*$ is a
Bayes-optimal denoiser and
\begin{equation}
    \mathcal{L}^*(X)
    =\mathbb{E}\lVert\epsilon-f_X^*(Z_K,K,X)\rVert^2.
    \label{eq:optimal_risk_expression}
\end{equation}

We now compare observation conditioning with history conditioning.  For brevity, write
\begin{equation}
    f_O^*:=f_O^*(Z_K,K,O),
    \qquad
    f_H^*:=f_H^*(Z_K,K,H).
\end{equation}
Because $O$ is a measurable function of $H$, any function of $(Z_K,K,O)$ is also a
function of $(Z_K,K,H)$.  In particular, $f_H^*-f_O^*$ is measurable with respect to
$\sigma(Z_K,K,H)$.  Decompose the observation-conditioned residual as
\begin{equation}
    \epsilon-f_O^*
    =(\epsilon-f_H^*)+(f_H^*-f_O^*).
    \label{eq:reactive_residual_decomposition}
\end{equation}
After squaring and taking expectations, we obtain
\begin{align}
    \mathcal{L}^*(O)
    &=\mathbb{E}\lVert\epsilon-f_H^*\rVert^2
      +\mathbb{E}\lVert f_H^*-f_O^*\rVert^2 \notag\\
    &\quad
      +2\mathbb{E}\!\left[(\epsilon-f_H^*)^\top(f_H^*-f_O^*)\right].
    \label{eq:nested_risk_expansion}
\end{align}
Again, the cross term is zero.  To see this, condition on all the information available
to the history-conditioned denoiser:
\begin{align}
    &\mathbb{E}\!\left[(\epsilon-f_H^*)^\top(f_H^*-f_O^*)
        \mid Z_K,K,H\right] \notag\\
    &\quad=
      \mathbb{E}[\epsilon-f_H^*\mid Z_K,K,H]^\top
      (f_H^*-f_O^*)=0.
    \label{eq:nested_cross_term_zero}
\end{align}
Taking the outer expectation preserves zero.  Moreover,
$\mathbb{E}\lVert\epsilon-f_H^*\rVert^2=\mathcal{L}^*(H)$ by
\eqref{eq:optimal_risk_expression}.  Consequently,
\begin{equation}
    \mathcal{L}^*(O)
    =\mathcal{L}^*(H)
     +\mathbb{E}\!\left[
        \left\lVert
        f_H^*(Z_K,K,H)-f_O^*(Z_K,K,O)
        \right\rVert^2
     \right].
    \label{eq:nested_pythagorean_identity}
\end{equation}
Rearranging gives
\begin{equation}
    \mathcal{L}^*(O)-\mathcal{L}^*(H)
    =\mathbb{E}\!\left[
        \left\lVert
        f_H^*(Z_K,K,H)-f_O^*(Z_K,K,O)
        \right\rVert^2
    \right]\geq 0.
    \label{eq:denoising_risk_gap}
\end{equation}
The final inequality follows because a squared Euclidean norm is nonnegative.  Hence
additional conditioning information cannot increase the Bayes risk.  Equality holds if
and only if $f_H^*(Z_K,K,H)=f_O^*(Z_K,K,O)$ almost surely.
\end{proof}

\subsection{Distributional aliasing strictly decreases the optimal loss.}
\paragraph{Proof structure.}
Proposition~\ref{proposition1} already shows that the two optimal losses are equal exactly
when the two optimal denoisers are equal.  The only technically difficult step is to show
that equality of the denoisers is equivalent to equality of the complete conditional action
distributions.  We isolate that step in the following lemma.  Proposition~\ref{proposition2}
then follows from the lemma and the definition of conditional mutual information.

\begin{lemma}[The optimal denoiser identifies the conditional action distribution]
\label{lem:denoiser_identifiability}
Under the assumptions above,
\begin{equation}
    f_H^*(Z_K,K,H)=f_O^*(Z_K,K,O)
    \quad\text{almost surely}
\end{equation}
if and only if
\begin{equation}
    P_E(A\in\cdot\mid H)=P_E(A\in\cdot\mid O)
    \quad\text{almost surely}.
    \label{eq:conditional_independence_kernel}
\end{equation}
\end{lemma}

\begin{proof}
We prove the two directions separately.  The forward direction, from equal denoisers to
equal conditional action distributions, contains the main technical argument and has four
steps.  First, we select a diffusion step that contains both clean-action signal and
Gaussian noise.  Second, we derive an explicit formula connecting the optimal denoiser
to the conditional density of the noisy action.  Third, we use this formula to show that
equal denoisers imply equal conditional distributions of the noisy action.  Fourth, we
show that equality of the noisy-action distributions implies equality of the clean-action
distributions.  The reverse direction is Step~5 and follows by applying the same diffusion
corruption to equal clean-action distributions.

\paragraph{Step 1: Select a diffusion step containing both signal and noise.}
Suppose first that the two optimal denoisers are equal almost surely.  We restrict this
equality to one sampled diffusion step at which the clean action has a nonzero coefficient
and the Gaussian noise has a positive coefficient.
Choose a diffusion step $k$ such that
\begin{equation}
    P(K=k)>0,\qquad \alpha_k\neq0,\qquad \sigma_k>0.
    \label{eq:chosen_nondegenerate_step}
\end{equation}
Such a step exists by assumption.  Since $P(K=k)>0$, equality of the denoisers in
the statement of the lemma implies
\begin{equation}
    \mathbb{E}[\epsilon\mid Z_k,H]
    =\mathbb{E}[\epsilon\mid Z_k,O]
    \quad\text{almost surely}.
    \label{eq:fixed_step_denoiser_equality}
\end{equation}

\paragraph{Step 2: Relate each optimal denoiser to its noisy-action density.}
In this step, we derive the optimal noise prediction from the conditional density of the
noisy action, first under history conditioning and then under observation conditioning.
We next express both sides of \eqref{eq:fixed_step_denoiser_equality} in terms of the
conditional density of $Z_k$.  Fix a history value $h$, and let
\begin{equation}
    \mu_h(B):=P_E(A\in B\mid H=h)
\end{equation}
be the conditional action distribution for each measurable set
$B\subseteq\mathbb{R}^d$.  Define the Gaussian density
\begin{equation}
    g_{\sigma_k}(x)
    :=\frac{1}{(2\pi\sigma_k^2)^{d/2}}
      \exp\!\left(-\frac{\lVert x\rVert^2}{2\sigma_k^2}\right).
    \label{eq:gaussian_density}
\end{equation}
Conditional on $A=a$ and $H=h$, the random variable
$Z_k=\alpha_k A+\sigma_k\epsilon$ has density
$g_{\sigma_k}(z-\alpha_k a)$.  Averaging this density over
$A\mid H=h$ gives the density of $Z_k\mid H=h$:
\begin{equation}
    q_h(z)
    =\int_{\mathbb{R}^d}
       g_{\sigma_k}(z-\alpha_k a)\,\mu_h(da).
    \label{eq:noisy_density_given_history}
\end{equation}
The integrand is strictly positive for every $z$ and $a$, so
$q_h(z)>0$ for every $z$.  Moreover, differentiating the Gaussian density gives
\begin{equation}
    \nabla_z g_{\sigma_k}(z-\alpha_k a)
    =-\frac{z-\alpha_k a}{\sigma_k^2}
      g_{\sigma_k}(z-\alpha_k a).
    \label{eq:gaussian_density_derivative}
\end{equation}
Using \eqref{eq:gaussian_density_derivative} inside
\eqref{eq:noisy_density_given_history}, we obtain
\begin{align}
    \nabla_z q_h(z)
    &=-\frac{1}{\sigma_k^2}
      \int_{\mathbb{R}^d}(z-\alpha_k a)
      g_{\sigma_k}(z-\alpha_k a)\,\mu_h(da) \notag\\
    &=-\frac{q_h(z)}{\sigma_k^2}
      \left(z-\mathbb{E}[\alpha_k A\mid Z_k=z,H=h]\right).
    \label{eq:noisy_density_derivative}
\end{align}
The second equality follows from the conditional density of $A$ after observing
$Z_k=z$ and $H=h$: for every integrable function $r$,
\begin{equation}
    \mathbb{E}[r(A)\mid Z_k=z,H=h]
    =\frac{
        \int_{\mathbb{R}^d}r(a)
        g_{\sigma_k}(z-\alpha_k a)\,\mu_h(da)
      }{q_h(z)}.
    \label{eq:conditional_action_given_noisy_action}
\end{equation}
Dividing \eqref{eq:noisy_density_derivative} by the positive number $q_h(z)$ gives
\begin{equation}
    \nabla_z\log q_h(z)
    =-\frac{1}{\sigma_k^2}
      \left(z-\mathbb{E}[\alpha_k A\mid Z_k=z,H=h]\right).
    \label{eq:log_density_derivative_history}
\end{equation}
On the other hand, the forward diffusion equation implies
\begin{equation}
    \epsilon=\frac{Z_k-\alpha_k A}{\sigma_k}.
\end{equation}
Therefore,
\begin{align}
    \mathbb{E}[\epsilon\mid Z_k=z,H=h]
    &=\frac{z-\mathbb{E}[\alpha_k A\mid Z_k=z,H=h]}{\sigma_k} \notag\\
    &=-\sigma_k\nabla_z\log q_h(z).
    \label{eq:history_denoiser_density_derivative}
\end{align}

The same calculation can be performed while conditioning only on $O=o$.  Let
\begin{equation}
    \mu_o(B):=P_E(A\in B\mid O=o)
\end{equation}
and define
\begin{equation}
    q_o(z)
    :=\int_{\mathbb{R}^d}
       g_{\sigma_k}(z-\alpha_k a)\,\mu_o(da).
    \label{eq:noisy_density_given_observation}
\end{equation}
Then $q_o(z)>0$ for every $z$, and the identical calculation yields
\begin{equation}
    \mathbb{E}[\epsilon\mid Z_k=z,O=o]
    =-\sigma_k\nabla_z\log q_o(z).
    \label{eq:observation_denoiser_density_derivative}
\end{equation}

\paragraph{Step 3: Use equal denoisers to obtain equal noisy-action distributions.}
The formulas from Step~2 turn equality of the two denoisers into equality of derivatives
of the two log densities.  We then show that two normalized densities with these equal
derivatives must themselves be equal.
We now return to the equality of the denoisers.  For almost every history $h$, let
$o$ be the current-observation component of $h$.  Equations
\eqref{eq:fixed_step_denoiser_equality},
\eqref{eq:history_denoiser_density_derivative}, and
\eqref{eq:observation_denoiser_density_derivative} imply
\begin{equation}
    \nabla_z\log q_h(z)=\nabla_z\log q_o(z)
    \quad\text{for }q_h(z)\,dz\text{-almost every }z.
    \label{eq:conditional_log_derivatives_equal}
\end{equation}
Because $q_h(z)>0$ everywhere, an equality holding for
$q_h(z)\,dz$-almost every $z$ also holds for ordinary Lebesgue-almost every $z$.
Both sides of \eqref{eq:conditional_log_derivatives_equal} are continuous in $z$;
hence they are equal for every $z$.  It follows that
\begin{equation}
    \nabla_z\!\left(\log q_h(z)-\log q_o(z)\right)=0
    \quad\text{for every }z\in\mathbb{R}^d.
\end{equation}
Since $\mathbb{R}^d$ is connected, a differentiable function with zero gradient
everywhere is constant.  Thus, for some constant $c_h>0$,
\begin{equation}
    q_h(z)=c_h q_o(z)
    \quad\text{for every }z.
\end{equation}
Both sides are probability densities, so their integrals over $\mathbb{R}^d$ equal one:
\begin{equation}
    1=\int q_h(z)\,dz
     =c_h\int q_o(z)\,dz
     =c_h.
\end{equation}
Consequently,
\begin{equation}
    q_h(z)=q_o(z)
    \quad\text{for every }z
    \quad\text{and for almost every }h.
    \label{eq:noisy_conditional_densities_equal}
\end{equation}

\paragraph{Step 4: Recover equality of the clean-action distributions.}
Step~3 shows that the two clean-action distributions become equal after the same
Gaussian noise is added.  We now show explicitly that this can happen only when the
two clean-action distributions were already equal.
It remains to show that equality after adding Gaussian noise implies equality before
adding noise.  For any random vector $Y\in\mathbb{R}^d$, define
\begin{equation}
    \varphi_Y(u):=\mathbb{E}[e^{i u^\top Y}],
    \qquad u\in\mathbb{R}^d.
    \label{eq:characteristic_function_definition}
\end{equation}
This function uniquely determines the probability distribution of $Y$.  From
\eqref{eq:noisy_conditional_densities_equal}, the conditional distributions of
$Z_k\mid H=h$ and $Z_k\mid O=o$ are equal, and hence
\begin{equation}
    \mathbb{E}[e^{iu^\top Z_k}\mid H=h]
    =\mathbb{E}[e^{iu^\top Z_k}\mid O=o]
    \quad\text{for every }u\in\mathbb{R}^d.
    \label{eq:noisy_fourier_expectations_equal}
\end{equation}
Using $Z_k=\alpha_k A+\sigma_k\epsilon$ and the independence of $\epsilon$ from
$(A,H,O)$, the left-hand side is
\begin{align}
    \mathbb{E}[e^{iu^\top Z_k}\mid H=h]
    &=\mathbb{E}[e^{i\alpha_k u^\top A}\mid H=h]
      \mathbb{E}[e^{i\sigma_k u^\top\epsilon}] \notag\\
    &=\varphi_{A\mid h}(\alpha_k u)
      \exp\!\left(-\frac{\sigma_k^2\lVert u\rVert^2}{2}\right).
    \label{eq:noisy_fourier_expectation_history}
\end{align}
Similarly,
\begin{equation}
    \mathbb{E}[e^{iu^\top Z_k}\mid O=o]
    =\varphi_{A\mid o}(\alpha_k u)
      \exp\!\left(-\frac{\sigma_k^2\lVert u\rVert^2}{2}\right).
    \label{eq:noisy_fourier_expectation_observation}
\end{equation}
Substituting \eqref{eq:noisy_fourier_expectation_history} and
\eqref{eq:noisy_fourier_expectation_observation} into
\eqref{eq:noisy_fourier_expectations_equal} gives
\begin{equation}
    \varphi_{A\mid h}(\alpha_k u)
    e^{-\sigma_k^2\lVert u\rVert^2/2}
    =
    \varphi_{A\mid o}(\alpha_k u)
    e^{-\sigma_k^2\lVert u\rVert^2/2}
    \quad\text{for every }u.
    \label{eq:action_fourier_expectations_with_gaussian_factor}
\end{equation}
The factor $e^{-\sigma_k^2\lVert u\rVert^2/2}$ is strictly positive, so it can be
cancelled.  Because $\alpha_k\neq0$, every $v\in\mathbb{R}^d$ can be written as
$v=\alpha_k u$.  We therefore obtain
\begin{equation}
    \varphi_{A\mid h}(v)=\varphi_{A\mid o}(v)
    \quad\text{for every }v\in\mathbb{R}^d.
\end{equation}
Since the function in \eqref{eq:characteristic_function_definition} uniquely determines
a probability distribution, this proves
\begin{equation}
    P_E(A\in\cdot\mid H=h)=P_E(A\in\cdot\mid O=o)
\end{equation}
for almost every $h$, which is \eqref{eq:conditional_independence_kernel}.

\paragraph{Step 5: Prove the reverse direction.}
We finally assume that the conditional clean-action distributions are equal and show
that the corresponding optimal denoisers must be equal.  This direction does not require
the density calculations above.
Suppose
\eqref{eq:conditional_independence_kernel} holds.  Then, for almost every history $h$
and its current observation $o$, the conditional laws of $A\mid H=h$ and
$A\mid O=o$ are equal.  Since $\epsilon$ is independent of $(A,H,O)$, the conditional
laws of $(A,\epsilon)$ are also equal under these two conditioning events.  Applying the
same deterministic map
\begin{equation}
    (a,e)\longmapsto (\alpha_k a+\sigma_k e,e)
\end{equation}
shows that the conditional joint laws of $(Z_k,\epsilon)$ are equal for every sampled
diffusion step $k$.  Equal joint laws give equal conditional means, so
\begin{equation}
    \mathbb{E}[\epsilon\mid Z_k,H]
    =\mathbb{E}[\epsilon\mid Z_k,O]
\end{equation}
for every diffusion step $k$ of positive sampling probability.  Hence
$f_H^*(Z_K,K,H)=f_O^*(Z_K,K,O)$ almost surely, which proves the reverse direction.
Steps~1--4 prove the forward direction and Step~5 proves the reverse direction, so the
two statements in the lemma are equivalent.
\end{proof}

\begin{proposition}
\label{proposition2}
Under the assumptions above, history conditioning strictly reduces the optimal
population diffusion loss if and only if the available history contains additional
information about the expert action:
\begin{equation}
    \mathcal{L}^*(H)<\mathcal{L}^*(O)
    \quad\Longleftrightarrow\quad
    I_E(A;H\mid O)>0.
    \label{eq:strict_equivalence}
\end{equation}
In particular, the strict inequality holds under the history-resolvable aliasing condition
in \cref{def:history_resolvable_aliasing}, even when the conditional action laws have
identical means and variances.
\end{proposition}

\begin{proof}
We first characterize when the two losses are equal.  From
\eqref{eq:denoising_risk_gap},
\begin{equation}
    \mathcal{L}^*(H)=\mathcal{L}^*(O)
    \quad\Longleftrightarrow\quad
    f_H^*(Z_K,K,H)=f_O^*(Z_K,K,O)
    \quad\text{almost surely}.
    \label{eq:risk_equality_denoiser_equality}
\end{equation}
By \cref{lem:denoiser_identifiability}, this is equivalent to
\begin{equation}
    P_E(A\in\cdot\mid H)=P_E(A\in\cdot\mid O)
    \quad\text{almost surely}.
\end{equation}

On the other hand, by \eqref{eq:conditional_mi},
\begin{equation}
    I_E(A;H\mid O)
    =\mathbb{E}\!\left[
        \operatorname{KL}\!\left(
        P_E(A\in\cdot\mid H)
        \,\middle\|\,
        P_E(A\in\cdot\mid O)
        \right)
    \right].
\end{equation}
The quantity inside this expectation is nonnegative and equals zero exactly when its
two probability distributions are equal.  Therefore,
\begin{equation}
    I_E(A;H\mid O)=0
    \quad\Longleftrightarrow\quad
    P_E(A\in\cdot\mid H)=P_E(A\in\cdot\mid O)
    \quad\text{almost surely}.
    \label{eq:zero_information_equal_action_laws}
\end{equation}
Combining these equivalences gives
\begin{equation}
    \mathcal{L}^*(H)=\mathcal{L}^*(O)
    \quad\Longleftrightarrow\quad
    I_E(A;H\mid O)=0.
\end{equation}
Taking the complementary cases and using
$\mathcal{L}^*(H)\leq\mathcal{L}^*(O)$ and $I_E(A;H\mid O)\geq0$ proves
\begin{equation}
    \mathcal{L}^*(H)<\mathcal{L}^*(O)
    \quad\Longleftrightarrow\quad
    I_E(A;H\mid O)>0.
\end{equation}
\end{proof}

\paragraph{Scope of the result.}
The propositions compare optimal population diffusion training losses under the expert
data distribution.  They do not by themselves guarantee lower finite-sample error,
successful optimization, or higher closed-loop task return.  Those conclusions require
additional assumptions on model capacity, estimation, optimization, and distribution
shift.  The result isolates the precise distributional condition under which the
available history can, in principle, improve diffusion-based imitation.






\newpage

\end{document}